\documentclass{article}
\usepackage{uai2019}
\usepackage[margin=1in]{geometry}

\usepackage{microtype}
\usepackage{booktabs} 
\usepackage{hyperref}
\usepackage{url}
\usepackage{graphicx}
\usepackage{algorithm}
\usepackage{algorithmic}
\usepackage{subcaption}
\usepackage{float}
\usepackage{wrapfig}
\usepackage{xcolor}
\usepackage{natbib}
\usepackage{todonotes}
\usepackage{dblfloatfix} 

\usepackage{amsmath,amsfonts,bm}
\usepackage{amsthm}

\def\eqref#1{equation~\ref{#1}}

\def\1{\bm{1}}

\DeclareMathAlphabet{\mathsfit}{\encodingdefault}{\sfdefault}{m}{sl}
\SetMathAlphabet{\mathsfit}{bold}{\encodingdefault}{\sfdefault}{bx}{n}

\def\gX{{\mathcal{X}}}
\def\gY{{\mathcal{Y}}}
\def\gZ{{\mathcal{Z}}}

\newcommand{\E}{\mathbb{E}}

\newcommand{\R}{\mathbb{R}}

\DeclareMathOperator*{\argmax}{arg\,max}
\DeclareMathOperator*{\argmin}{arg\,min}

\DeclareMathOperator{\Tr}{Tr}

\newtheorem{thm}{Theorem}[section]

\newtheorem{lemma}[thm]{Lemma}

   \newcommand{\DS}{\mathrm{S}}

\usepackage{cleveref}

\title{Sinkhorn AutoEncoders}

\author{
{\bf Giorgio Patrini}\thanks{{ }Equal contributions. $^\bullet$Now at Deeptrace. $\dagger$Now at Google Brain. $\ddagger$Now at DeepMind.} $^\bullet$\\
UvA Bosch Delta Lab\\[0.5cm]
{\bf Samarth Bhargav}\\
University of Amsterdam\\
\And
{\bf Rianne van den Berg}$^* \dagger$\\
University of Amsterdam\\[0.5cm]
{\bf Max Welling}\\
University of Amsterdam\\
CIFAR\\
\And
{\bf Patrick Forr\'{e}}\\
University of Amsterdam\\[0.5cm]
{\bf Tim Genewein}$\ddagger$\\
Bosch Center for\\ Artificial Intelligence\\
\And
{\bf Marcello Carioni}\\
KFU Graz\\[0.5cm]
{\bf Frank Nielsen}\\
\'Ecole Polytechnique\\ Sony CSL
}

\begin{document}


\maketitle


\begin{abstract}
Optimal transport offers an alternative to maximum likelihood for learning generative autoencoding models. We show that minimizing the $p$-Wasserstein distance between the generator and the true data distribution is equivalent to the unconstrained min-min optimization of the 
$p$-Wasserstein distance between the encoder aggregated posterior and the prior in latent space, plus a reconstruction error. We also identify the role of its trade-off hyperparameter as the capacity of the generator: its Lipschitz constant. Moreover, we prove that optimizing the encoder over any class of universal approximators, such as deterministic neural networks, is enough to come arbitrarily close to the optimum. We therefore advertise this framework, which holds for any metric space and prior, as a sweet-spot of current generative autoencoding objectives.\\
We then introduce the Sinkhorn auto-encoder (SAE), which approximates and minimizes the $p$-Wasserstein distance in latent space via backprogation through the Sinkhorn algorithm. SAE directly works on samples, i.e. it models the aggregated posterior as an implicit distribution, with no need for a reparameterization trick for gradients estimations. SAE is thus able to work with different metric spaces and priors with minimal adaptations.\\
We demonstrate the flexibility of SAE on latent spaces with different geometries and priors and compare with other methods on benchmark data sets.
\end{abstract}

\section{INTRODUCTION}

Unsupervised learning aims at finding the underlying rules that govern a given data distribution $P_X$. It can be approached by learning to mimic the data generation process, 
or by finding an adequate representation of the data. 
Generative Adversarial Networks (GAN) \citep{Goodfellow2014gans} belong to the former class, by learning to transform noise into an implicit distribution that matches the given one. 
AutoEncoders (AE) \citep{hinton2006reducing} are of the latter type, by learning a representation that maximizes the mutual information between the data and its reconstruction, subject to an information bottleneck.
Variational AutoEncoders (VAE) \citep{kingma2013auto, rezende2014stochastic}, provide both a generative model --- i.e. a \emph{prior} distribution $P_Z$ on the latent space with a decoder $G(X|Z)$ that models the conditional likelihood --- and an encoder $Q(Z|X)$ --- approximating the \emph{posterior} distribution of the generative model. 
Optimizing the exact marginal likelihood is intractable in latent variable models such as VAE's.
Instead one maximizes the Evidence Lower BOund (ELBO) as a surrogate. This objective trades off a reconstruction error of the input distribution $P_X$ and a regularization term that aims at minimizing the Kullback-Leibler (KL) divergence from the approximate posterior $Q(Z|X)$ to the prior $P_Z$.

An alternative principle for learning generative autoencoders comes from the theory of Optimal Transport (OT) \citep{villani2008optimal}, where the usual KL-divergence KL$(P_X,P_G)$ is replaced by OT-cost divergences $W_c(P_X,P_G)$, among which the $p$-Wasserstein distances $W_p$ are proper metrics. In the papers \cite{tolstikhin2018wasserstein,bousquet2017vegan} it was shown that the objective $W_c(P_X,P_G)$ can be re-written as the minimization of the reconstruction error of the input $P_X$ over all \emph{probabilistic} encoders $Q(Z|X)$ constrained to the condition of matching the \emph{aggregated posterior} $Q_Z$ --- the average (approximate) posterior $\E_{P_X}[Q(Z|X)]$ --- to the prior $P_Z$ in the latent space. 
In Wasserstein AutoEncoders (WAE) \citep{tolstikhin2018wasserstein}, it was suggested, following the standard optimization principles, to softly enforce that constraint via a penalization term depending on a choice of a divergence $D(Q_Z,P_Z)$ in latent space. For any such choice of divergence this leads to the minimization of a lower bound of the original objective, leaving the question about the status of the original objective open.
Nonetheless, WAE empirically improves upon VAE for the two choices made there, namely either a Maximum Mean Discrepancy (MMD)  \citep{gretton2007kernel,sriperumbudur2010hilbert,sriperumbudur2011universality}, or an adversarial loss (GAN), again both in latent space. 

We contribute to the formal analysis of autoencoders with OT. 
First, using the Monge-Kantorovich equivalence \citep{villani2008optimal}, we show that (in non-degenerate cases) the objective $W_c(P_X,P_G)$ can be reduced to the minimization of the reconstruction error of $P_X$ over any class containing the class of all \emph{deterministic} encoders $Q(Z|X)$, again constrained to $Q_Z=P_Z$.\\
Second, when restricted to the $p$-Wasserstein distance $W_p(P_X,P_G)$, and by using a combination of triangle inequality and a form of data processing inequality for the generator $G$,  we show that the soft and \emph{unconstrained} minimization of the reconstruction error of $P_X$ together with the penalization term $\gamma \cdot W_p(Q_Z,P_Z)$ is actually an \emph{upper} bound to the original objective $W_p(P_X,P_G)$, where the regularization/trade-off hyperparameter $\gamma$ needs to match at least the \emph{capacity} of the generator $G$, i.e.\ its Lipschitz constant.
This suggests that using a $p$-Wasserstein metric $W_p(Q_Z,P_Z)$ in latent space in the WAE setting \citep{tolstikhin2018wasserstein} is a preferable choice.

Third, we show that the minimum of that objective can be approximated from above by any class of universal approximators for $Q(Z|X)$ to arbitrarily small error. In case we choose the L$_p$-norms $\|\cdot\|_p$ and corresponding $p$-Wasserstein distances $W_p$ one can use the results of \citep{hornik1991approximation} to show that any class of probabilistic encoders $Q(Z|X)$ that contains the class of \emph{deterministic} neural networks has all those desired properties. This justifies the use of such classes in practice. Note that analogous results for the latter for other divergences and function classes are unknown. 

Fourth, as a corollary we get the folklore claim that matching the aggregated posterior $Q_Z$ and prior $P_Z$ is a \emph{necessary} condition for learning the true data distribution $P_X$ in rigorous mathematical terms. Any deviation will thus be punished with a poorer performance.
Altogether, we have addressed and answered the open questions in \citep{tolstikhin2018wasserstein,bousquet2017vegan} in detail and highlighted the sweet-spot framework for generative autoencoder models based on Optimal Transport (OT) for \emph{any} metric space and \emph{any} prior distribution $P_Z$, and with special emphasis on Euclidean spaces, L$_p$-norms and neural networks.


The theory supports practical innovations. We are now in a position to learn deterministic autoencoders, $Q(Z|X)$, $G(X|Z)$, by minimizing a reconstruction error for $P_X$ and the $p$-Wasserstein distance on the latent space between samples of the aggregated posterior and the prior $W_p(\hat Q_Z,\hat P_Z)$. The computation of the latter is known to be difficult and costly (cp. Hungarian algorithm \citep{kuhn1955hungarian}). A fast approximate solution is provided by the \emph{Sinkhorn algorithm} \citep{cuturi}, which uses an entropic relaxation.
We follow \citep{frogner} and \citep{geneway2017learning}, by exploiting the differentiability of the Sinkhorn iterations, and unroll it for backpropagation. In addition we correct for the entropic bias of the Sinkhorn algorithm \citep{geneway2017learning, feydy2018otmmd}. Altogether, we call our method the \emph{Sinkhorn AutoEncoder (SAE)}. 


The Sinkhorn AutoEncoder is agnostic to the analytical form of the prior, as it optimizes a sample-based cost function which is aware of the geometry of the latent space. Furthermore, as a byproduct of using deterministic networks, it models the aggregated posterior as an implicit distribution \citep{mohamed2016learning} with no need of the reparametrization trick for learning the encoder \citep{kingma2013auto}.
Therefore, with essentially no change in the algorithm, we can learn models with normally distributed priors and aggregated posteriors, as well as distributions living on manifolds such as hyperspheres \citep{davidson2018hyperspherical} and probability simplices.

In our experiments we explore how well the Sinkhorn AutoEncoder performs on the benchmark datasets MNIST and CelebA with different prior distributions $P_Z$ and geometries in latent space, e.g.\ the Gaussian in Euclidean space or the uniform distribution on a hypersphere.
Furthermore, we compare the SAE to the VAE \citep{kingma2013auto}, to the WAE-MMD \citep{tolstikhin2018wasserstein} and other methods of approximating the $p$-Wasserstein distance in latent space like the Hungarian algorithm \citep{kuhn1955hungarian} and the Sliced Wasserstein AutoEncoder \citep{kolouri2018sliced}.
We also explore the idea of matching the aggregated posterior $Q_Z$ to a standard Gaussian prior $P_Z$ via the fact that the $2$-Wasserstein distance has a closed form for Gaussian distributions: we estimate the mean and covariance of $Q_Z$ on minibatches and use the loss $W_2(\hat Q_Z,P_Z)$ for backpropagation. Finally, we train SAE on MNIST with a probability simplex as a latent space and visualize the matching of the aggregate posterior and the prior.

\section{PRINCIPLES OF WASSERSTEIN AUTOENCODING}\label{principles}


\subsection{OPTIMAL TRANSPORT}

We follow \cite{tolstikhin2018wasserstein} and denote with $\gX, \gY, \gZ$ the sample spaces and with $X,Y,Z$ and $P_X, P_Y, P_Z$ the corresponding random variables and distributions. Given a map $F:\mathcal{X} \rightarrow \mathcal{Y}$ we denote by $F_\#$ the push-forward map acting on a distribution $P$ as $P \circ F^{-1}$. If $F(Y|X)$ is non-deterministic we define the push-forward $F(Y|X)_\# P_X$ of a distribution $P$ as the induced marginal of the joint distribution $F(Y|X)P_X$. 
For any measurable non-negative \emph{cost} $c : \gX \times \gY \to \R^{+} \cup \{ \infty\}$, one can define the following \emph{OT-cost} 
between distributions  $P_X$ and $P_Y$ via:
\begin{align}\label{kantorovich}
W_c(P_X,P_Y) &=  \inf_{ \Gamma \in \Pi(P_X,P_Y)} \E_{(X,Y) \sim \Gamma}[ c(X,Y) ],  
\end{align}
where $\Pi(P_X,P_Y)$ is the set of all joint distributions that have $P_X$ and $P_Y$ as the marginals.
The elements from $\Pi(P_X,P_Y)$ are called \emph{couplings} from $P_X$ to $P_Y$.
 If $c(x,y) = d(x,y)^p$ for a metric $d$ and $p \ge 1$ then $W_p:=\sqrt[p]{W_c}$ is called the \emph{$p$-Wasserstein distance}.

Let $P_X$ denote the true data distribution on $\gX$. We define a \emph{latent variable model} as follows: we fix a latent space $\gZ$ and a prior distribution $P_Z$ on $\gZ$ and consider the conditional distribution $G(X|Z)$ (the decoder) parameterized by a neural network $G$. Together they specify a generative model as $G(X|Z) P_Z$. The induced marginal will be denoted by $P_G$.
Learning $P_G$ such that it approximates the true $P_X$ is then defined as:
\begin{align}
 \min_G W_c(P_X,P_G).
\end{align}
Because of the infimum over $\Pi(P_X,P_G)$ inside $W_c$, this is intractable.
To rewrite this objective we consider the posterior distribution $Q(Z|X)$ (the encoder) and its \emph{aggregated posterior} $Q_Z$:
\begin{align}\label{aggregate}
Q_Z =  Q(Z|X)_\# P_X =  \E_{X \sim P_X} Q(Z|X),
\end{align}
the induced marginal of the joint $Q(Z|X)P_X$. 

\subsection{THE WASSERSTEIN AUTOENCODER (WAE)}

\cite{tolstikhin2018wasserstein} show that
if the decoder $G(X|Z)$ is deterministic, i.e. $P_G=G_\#P_Z$, or in other words, if all stochasticity of the generative model is captured by $P_Z$, then:
\begin{equation}\label{main}
W_c(P_X,P_G) = \inf_{\substack{Q(Z|X):\\{Q_Z = P_Z}}} \E_{X \sim P_X}\E_{Z \sim Q(Z|X)}[c(X,G(Z))]. 
\end{equation} 
Learning the generative model $G$ with the WAE amounts to the objective:
\begin{align}\label{objective}
 \min_G \min_{Q(Z|X)}  & \E_{X \sim P_X} \E_{Z \sim Q(Z|X)}[c(X,G(Z))] \nonumber \\
 + & \beta \cdot D(Q_Z,P_Z) ,
\end{align}
where $\beta >0$ is a Lagrange multiplier and $D$ is any divergence measure on probability distributions on $\gZ$. 
The specific choice for $D$  is left open. WAE uses either MMD \citep{gretton2012kernel} or a discriminator trained adversarially for $D$. As discussed in \cite{bousquet2017vegan}, Eq. \ref{objective} is a lower bound of Eq. \ref{main} for any choice of $D$ and any value of $\beta > 0$.
Minimizing this lower bound does not ensure a minimization of the original objective of Eq. \ref{main}.

\subsection{THEORETICAL CONTRIBUTIONS}


We improve upon the analysis of \cite{tolstikhin2018wasserstein} of generative autoencoders in the framework of Optimal Transport in several ways.
Our contributions can be summarized by the following theorem, upon which we will comment directly after.

\begin{thm}
\label{holy-grail-app}
Let $\gX$, $\gZ$ be endowed with any metrics and $p \ge 1$.
Let $P_X$ be a non-atomic distribution\footnote{A probability measure is non-atomic if every point in its support has zero measure. It is important to distinguish between the \emph{empirical} data distribution $\hat P_X$, which is always atomic, and the underlying \emph{true} distribution $P_X$, only which we need to assume to be non-atomic.} and $G(X|Z)$ be a deterministic generator/decoder that is $\gamma$-Lipschitz. Then we have the equality: 
\begin{align}
W_p(P_X,P_G) &= \inf_{Q \in \mathcal{F}} \sqrt[p]{\E_{X \sim P_X}\E_{Z \sim Q(Z|X)}\left[d(X,G(Z))^p\right]} \nonumber \\
& \quad + \gamma \cdot W_p(Q_Z,P_Z)\label{rhs-holy}\, , 
\end{align}
where $\mathcal{F}$ is any class of probabilistic encoders that at least contains a class of universal approximators.
If $\gX$, $\gZ$ are Euclidean spaces endowed with the L$_p$-norms $\|\cdot\|_p$ then a valid minimal choice for $\mathcal{F}$ is the class of all deterministic neural network encoders $Q$ (here written as a function), for which the objective reduces to:
\begin{align*}
W_p(P_X,P_G) =  &\inf_{Q \in \mathcal{F}} \sqrt[p]{\E_{X \sim P_X}\left[\|X - G(Q(X))\|_p^p\right]} \\ 
& \quad + \gamma \cdot W_p(Q_Z,P_Z).
\end{align*}
\end{thm}

The proof  of Theorem  \ref{holy-grail-app} can be found in Appendix \ref{proof:mainth1}, \ref{proof:consequence} and \ref{proof:propsuff}. It uses the following three arguments:

i.) It is the Monge-Kantorovich equivalence \citep{villani2008optimal} for non-atomic $P_X$ that allows us to restrict to \emph{deterministic} encoders $Q(Z|X)$. This is a first theoretical improvement over the Eq. \ref{main} from \cite{tolstikhin2018wasserstein}. 

ii.) The \emph{upper} bound can be achieved by a simple \emph{triangle inequality}:
$$ W_p(P_X,P_G) \le W_p(P_X, \tilde P_X) + W_p( \tilde P_X ,P_G),$$
where $\tilde P_X := G_\# Q(Z|X)_\# P_X = G_\# Q_Z$ is the reconstruction of $P_X$. Note that the triangle inequality is not available for other general cost functions or divergences. This might be a reason for the difficulty of getting upper bounds in such settings. 
On the other hand, if a divergence satisfies the triangle inequality 
then one can use the same argument to arrive at new variational optimization objectives and principles.

iii.) We then prove the \emph{data processing inequality} for the $W_p$-distance:
$$ W_p( G_\#Q_Z , G_\# P_Z) \le \gamma \cdot W_p(Q_Z,P_Z), $$
with any $\gamma \ge \|G\|_\text{Lip}$, the Lipschitz constant of $G$. 
Such an inequality is available and known for several other divergences usually with $\gamma=1$.

Putting all three pieces together we immediately arrive at the \emph{equality} (upper and lower bound) of the first part of Theorem \ref{holy-grail-app}. This insight directly suggests that using the divergence $W_p(Q_Z,P_Z)$ in latent space with a hyperparameter $\gamma \ge \|G\|_\text{Lip}$ in the WAE setting is a preferable choice. These are two further improvements over \cite{tolstikhin2018wasserstein}.
Note that if $G$ is a neural network with activation function $g$ with $\|g'\|_\infty \le 1$  (e.g.\ ReLU, sigmoid, $\tanh$, etc.)
and weight matrices $(B_\ell)_{ \ell =1, \dots, L}$, then $G$ is $\gamma$-Lipschitz for any $\gamma \ge \|B_1\|_p\cdots \|B_L\|_p$, where the latter is the product of the L$_p$-matrix norms (cp.  \cite{balan2017lipschitz}).

iv.) For the second part of Theorem \ref{holy-grail-app} we use the universal approximator property of neural networks \citep{hornik1991approximation} and the compatibility of the L$_p$-norm $\|\cdot \|_p$-norm with the $p$-Wasserstein distance $W_p$. Proving such statements for other divergences seems to require much more effort (if possible at all).

When the encoders are restricted to be neural networks of limited capacity, e.g.\ if their architecture is fixed,  then enforcing $Q_Z \approx P_Z$ might not be feasible in the general case of dimensionality mismatch  between $\gX$ and $\gZ$ \citep{rubenstein2018wasserstein}.
In fact, since the class of deterministic neural networks  (of limited capacity)
is much smaller than the class of deterministic measurable maps, one might consider adding noise to the output, i.e.\ use stochastic networks instead.
Nonetheless, neural networks can approximate any measurable map up to arbitrarily small error  \citep{hornik1991approximation}. 
Furthermore, in practice the encoder $Q(Z|X)$ maps from the high dimensional data space $\gX$ to the much lower dimensional latent space $\gZ$, suggesting that the task of matching distributions in the lower dimensional latent space $\gZ$ should be feasible.
Also, in view of Theorem \ref{holy-grail-app} it follows that learning deterministic autoencoders is sufficient to approach the theoretical bound and thus will be our empirical choice.

Theorem \ref{holy-grail-app} certifies that, \emph{failing} to match aggregated posterior and prior makes learning the data distribution impossible. Matching in latent space should be seen as fundamental as minimizing the reconstruction error, a fact known about the performance of VAE \citep{hoffman2016elbo, higgins2016beta, alemi2018fixing, rosca2018distribution}. This necessary condition for learning the data distribution turns out to be also sufficient assuming that the set of encoders is expressive enough to nullify the reconstruction error.  

With the help of Theorem \ref{holy-grail-app} we arrive at the following \emph{unconstrained min-min-optimization} objective over deterministic decoder and encoder neural networks ($Q$ written as a function here): 
\begin{align*}
\min_G \min_{Q} &\sqrt[p]{\E_{X \sim P_X} \left[\|X-G(Q(X))\|_p^p\right]}\\
& + \gamma \cdot W_p(Q_Z,P_Z),
\end{align*}
with $\gamma \ge \|G\|_\text{Lip}$ for all occuring $G$.


\section{THE SINKHORN AUTOENCODER}
\subsection{ENTROPY REGULARIZED OPTIMAL TRANSPORT}

Even though the theory supports the use of the $p$-Wasserstein distance $W_p(Q_Z,P_Z)$ in latent space, it is notoriously hard to compute or estimate.
In practice, we will need to approximate $W_p(Q_Z,P_Z)$ via samples from $Q_Z$ (and $P_Z$).
The sample version $W_p(\hat Q_Z, \hat P_Z)$ with $\hat P_Z= \frac{1}{M} \sum_{m=1}^M \delta_{z_m}$ and $\hat Q_Z= \frac{1}{M} \sum_{m=1}^M \delta_{ \tilde z_m}$ has an exact solution, which can be computed using the Hungarian algorithm \citep{kuhn1955hungarian} in near $O(M^3)$ time (time complexity).
Furthermore, $W_p(\hat Q_Z, \hat P_Z)$ will differ from $W_p(Q_Z, P_Z)$ in size of about $O(M^{-\frac{1}{k}})$ (sample complexity), where $k$ is the dimension of $\gZ$ \citep{weed2017sharp}. Both complexity measures are unsatisfying in practice, but they can be improved via \emph{entropy regularization} \citep{cuturi}, which we will explain next. 

Following \cite{geneway2017learning, genevay2019sample, feydy2018otmmd} we define the \emph{entropy regularized OT cost} with $\varepsilon \ge 0$:
\begin{align}
\label{reg-ot-prob}
\tilde{S}_{c,\varepsilon}(P_X,P_Y) &:=  \inf_{ \Gamma \in \Pi(P_X,P_Y)} \E_{(X,Y) \sim \Gamma}[ c(X,Y) ]  \nonumber\\
& \qquad + \varepsilon \cdot \mathrm{KL}(\Gamma,P_X \otimes P_Y).  
\end{align}
This is in general not a divergence due to its entropic bias. When we remove this bias we arrive at the \emph{Sinkhorn divergence}:
\begin{align}
\label{reg-ot-prob-unbiased}
S_{c,\varepsilon}(P_X,&P_Y) := \tilde{S}_{c,\varepsilon}(P_X,P_Y) \nonumber\\
& - \frac{1}{2}\left( \tilde{S}_{c,\varepsilon}(P_X,P_X) + \tilde{S}_{c,\varepsilon}(P_Y,P_Y) \right).
\end{align}
The Sinkhorn divergence 
has the following limiting behaviour:
\[\begin{array}{rclcl}
S_{c,\varepsilon}(P_X,P_Y) & \stackrel{\varepsilon \to 0}{\longrightarrow} & W_{c}(P_X,P_Y), \\
S_{c,\varepsilon}(P_X,P_Y) & \stackrel{\varepsilon \to \infty}{\longrightarrow} & \mathrm{MMD}_{-c}(P_X,P_Y). 
\end{array}\]
This means that the Sinkhorn divergence $S_{c,\varepsilon}$ interpolates between OT-divergences and MMDs \citep{gretton2012kernel}. 
On the one hand, for small $\varepsilon$ it is known that $S_{c,\varepsilon}$ deviates from the initial objective $W_c$ by about $O(\varepsilon \log(1/\varepsilon))$ \citep{genevay2019sample}. 
On the other hand, if $\varepsilon$ is big enough then $S_{c,\varepsilon}$ will have the  more favourable sample complexity of $O(M^{-\frac{1}{2}})$ of MMDs, which is independent of the dimension, and was proven in \cite{genevay2019sample}.
Furthermore, the \emph{Sinkhorn algorithm} \citep{cuturi}, which will be explained in the section \ref{sinkhorn-sinkhorn}, allows for faster computation of the Sinkhorn divergence $S_{c,\varepsilon}$ with time complexity close to $O(M^2)$ \citep{altschuler2017near}.
Therefore, if we balance $\varepsilon$ well, we are close to our original objective and at the same time have favourable computational and statistical properties.

\subsection{THE SINKHORN AUTOENCODER OBJECTIVE}

Guided by the theoretical insights, we can restrict the WAE framework \citep{tolstikhin2018wasserstein} to Sinkhorn divergences with cost $\tilde c$ in latent space and $c$ in data space to arrive at the  objective:
\begin{align}\label{sae-objective}
 \min_G \min_{Q(Z|X)}  & \E_{X \sim P_X} \E_{Z \sim Q(Z|X)}[c(X,G(Z))] \nonumber \\
  &\qquad + \beta \cdot S_{\tilde c,\varepsilon}(Q_Z,P_Z) ,
\end{align}
with hyperparameters $\beta \ge 0$ and $\varepsilon \ge 0$.\\
Restricting further to $p$-Wasserstein distances, corresponding Sinkhorn divergences
and deterministic en-/decoder neural networks, we arrive at the \emph{Sinkhorn AutoEncoder (SAE)} objective:
 \begin{align}\label{sae-objective-strict}
 \min_G \min_{Q}  & \sqrt[p]{\E_{X \sim P_X}\left[\|X-G(Q(X))\|_p^p\right]} \nonumber\\
  & \qquad +\gamma \cdot S_{\|\cdot\|_p^p,\varepsilon}(Q_Z,P_Z)^{\frac{1}{p}},
\end{align}
which is then up to the $\varepsilon$-terms close to the original objective.
Note that for computational reasons 
it is sometimes convenient to remove the $p$-th roots again. 
The inequality $ \sqrt[p]{a} + \sqrt[p]{b} \le 2 \sqrt[p]{a +b}$ shows that the additional loss is small, while still minimizing an upper bound 
(using $\beta:=\gamma^p$). 


\subsection{THE SINKHORN ALGORITHM}
\label{sinkhorn-sinkhorn}
\label{sec:sinkhorn}

Now that we have the general Sinkhorn AutoEncoder optimization objective, we need to review how the Sinkhorn divergence 
$S_{\tilde c,\varepsilon}(Q_Z,P_Z)$ can be estimated in practice by the \emph{Sinkhorn algorithm} \citep{cuturi} using samples.

If we take $M$ samples each from $Q_Z$ and $P_Z$, we get the corresponding empirical (discrete) distributions concentrated on $M$ points:
$\hat P_Z  = \frac{1}{M}\sum_{m=1}^M \delta_{z_m}$ and $\hat Q_Z = \frac{1}{M} \sum_{m=1}^M \delta_{\tilde z_m}$.
Then, the optimal coupling of the (empirical) entropy regularized OT-cost $\tilde S_{\tilde c,\varepsilon}(\hat Q_Z, \hat P_Z)$  with $\varepsilon \ge 0$ is given by the matrix:%
\begin{align}\label{eq:discrete-sink-argmin}
R^*   & := \argmin_{R \in \DS_M} \tfrac{1}{M} \langle R , \tilde C\rangle_F - \varepsilon \cdot H(R),
\end{align}
%
where $\tilde C_{ij} = \tilde c(\tilde z_i,z_j)$ is the matrix associated to the cost $\tilde c$, $R$ is a doubly stochastic matrix as defined in
$\DS_M = \{ R \in \R^{M \times M}_{\ge 0} \;|\;   R\textbf{1}= \textbf{1} ,\, R^T\textbf{1}= \textbf{1} \}, $
and $\langle \cdot, \cdot \rangle_F$ denotes the Frobenius inner product; $\textbf{1}$ is the vector of ones and $H(R) = -\sum_{i, j=1}^{M} R_{i,j} \log R_{i,j}$ is the entropy of $R$.
\setlength{\textfloatsep}{0.35cm}

\begin{algorithm}[t!]\small
\caption{\textsc{Sinkhorn}}\label{algo:sinkhorn}
\begin{algorithmic}
\STATE \textbf{Input:} $\{\tilde z_i\}_{i=1}^M \sim Q_Z$, $\{z_j\}_{j=1}^M \sim P_Z$, $\varepsilon$, $L$\\
\STATE $\forall {i,j}:~\tilde C_{ij} = \tilde c(\tilde z_i, z_j)$
\STATE $K = \exp({-\tilde C/\varepsilon})$, $u \leftarrow \textbf{1}$ \qquad  \# elem-wise exp\\
\STATE \textbf{repeat  until convergence, but at most} $L$ \textbf{times}:\\
\STATE \quad $v \leftarrow \textbf{1} / (K^\top u)$ \qquad ~~~ \# elem-wise division\\
\STATE \quad $u \leftarrow \textbf{1} / (Kv)$ \\ 
\STATE $R^* \leftarrow \mbox{Diag}(u)\,K\,\mbox{Diag}(v)$ \qquad ~~~ \# plus rounding step\\ 
\STATE \textbf{Output:} $R^*$, $\tilde C$. 
\end{algorithmic}
\end{algorithm}


\cite{cuturi} shows that the \textsc{Sinkhorn} Algorithm \ref{algo:sinkhorn} \citep{sinkhorn} returns its $\varepsilon$-regularized optimum $R^*$ (see Eq. \ref{eq:discrete-sink-argmin}) in the limit $L \to \infty$, which is also unique due to strong convexity of the entropy. The Sinkhorn algorithm is a fixed point algorithm that is much faster than the Hungarian algorithm: it runs in nearly $O(M^2)$ time \citep{altschuler2017near} and can be efficiently implemented with matrix multiplications; see Algorithm \ref{algo:sinkhorn}.
For better differentiability properties we deviate from Eq. \ref{reg-ot-prob-unbiased} and use
the \emph{unbiased sharp Sinkhorn loss} \citep{luise2018differential,geneway2017learning} by dropping the entropy terms (only) in the evaluations:
%
\begin{align}\label{eq:unbiased-sharp-sinkhorn}
\overline S_{\tilde c,\varepsilon}(\hat Q_Z, &\hat P_Z)    :=  \frac{1}{M}\langle R^* , \tilde C\rangle_F \nonumber\\
& - \frac{1}{2M}\left( \langle R^*_{\hat Q_Z} , \tilde C_{\hat Q_Z}\rangle_F + \langle R^*_{\hat P_Z} , \tilde C_{\hat P_Z}\rangle_F   \right),
\end{align}
where the indices $\hat Q_Z$, $\hat P_Z$ refer to Eq. \ref{eq:discrete-sink-argmin} applied to the samples from $Q_Z$ in both arguments
 and then $P_Z$ in both arguments, respectively.

Since this only deviates from Eq. \ref{reg-ot-prob-unbiased} in $\varepsilon$-terms we still have all the mentioned properties, e.g.\ that the optimum of this Sinkhorn distance approaches the optimum of the OT-cost with the stated rate \citep{geneway2017learning,cominetti1994asymptotic, weed2018explicit}. 
Furthermore, for numerical stability we use the Sinkhorn algorithm in log-space \citep{chizat2016scaling,schmitzer2016stabilized}. 
In order to round the $R$ that results from a finite number $L$ of Sinkhorn iterations to a doubly stochastic matrix, we use the procedure described Algorithm 2 of \citep{altschuler2017near}.

The smaller the $\varepsilon$, the smaller the entropy and the better the approximation of the OT-cost.  At the same time, a larger number of steps $O(L)$ is needed to converge, while the rate of convergence remains linear in $L$ \citep{geneway2017learning}.
Note that all Sinkhorn operations are differentiable. Therefore, when the distance is used as a cost function, we can unroll $O(L)$ iterations and backpropagate \citep{geneway2017learning}. 
In conclusion, we obtain a differentiable surrogate for OT-cost between empirical distributions; the approximation arises from sampling, entropy regularization and the finite amount of steps in place of convergence.

\subsection{TRAINING THE SINKHORN AUTOENCODER}

To train the Sinkhorn AutoEncoder with encoder $Q_A$, decoder $G_B$ and with weights $A$, $B$, resp., we sample minibatches $x=\{x_i\}_{i=1}^M$ from the data distribution $P_X$ and $z=\{z_i\}_{i=1}^M$ from the prior $P_Z$. After encoding $x$ we then run the \textsc{Sinkhorn} Algorithm \ref{algo:sinkhorn} three times 
(for $(x,z)$, $(x,x)$ and $(z,z)$) to find the optimal couplings and then compute the unbiased \textsc{SinkhornLoss} via Eq. 
\ref{eq:unbiased-sharp-sinkhorn}. Note that the $L$ \textsc{Sinkhorn} steps in Algorithm \ref{algo:sinkhorn} are differentiable.
 The weights can then be updated via (auto-)differentiation through the \textsc{Sinkhorn} steps (together with the gradient of the reconstruction loss). One training round is summarized in Algorithm \ref{algo:sae-training}. 
\begin{algorithm}[t!]\small
\caption{\textsc{SAE Training round}}\label{algo:sae-training}
\begin{algorithmic}
\STATE \textbf{Input:} encoder weights $A$, decoder weights $B$, $\varepsilon$, $L$, $\beta$
\STATE \textbf{Minibatch:} $x=\{x_i\}_{i=1}^M \sim P_X$, $z=\{z_j\}_{j=1}^M \sim P_Z$ \\
\STATE $\tilde z \leftarrow Q_A(x)$, $\tilde x \leftarrow G_B(\tilde z)$ \\
\STATE $D = \frac{1}{M}\|x - \tilde x\|^p_p $ \\
\STATE $S = \textsc{SinkhornLoss}(\tilde z, z,\varepsilon,L)$,~~~ \# 3\,x\,Alg.\,\ref{algo:sinkhorn} \!+\! Eq.\,\ref{eq:unbiased-sharp-sinkhorn} \\
\STATE \textbf{Update:} $A,B$ with gradient $\nabla_{(A,B)} (D + \beta \cdot S)$.
\end{algorithmic}
\end{algorithm}
Small $\varepsilon$ and large $L$ worsen the numerical stability of the Sinkhorn. In most experiments, both $c$ and $\tilde c$ will be $\|\cdot \|_2^2$.
Experimentally we found that the re-calculation of the three optimal couplings at each iteration is not a significant overhead. 

SAE can in principle work with arbitrary priors. The only requirement coming from the Sinkhorn is the ability to generate samples. The choice should be motivated by the desired geometric properties of the latent space.

\begin{figure*}[ht]
\vspace{-12pt}
 \centering
\begin{subfigure}{.18\textwidth}
  \centering
  \includegraphics[width=\linewidth]{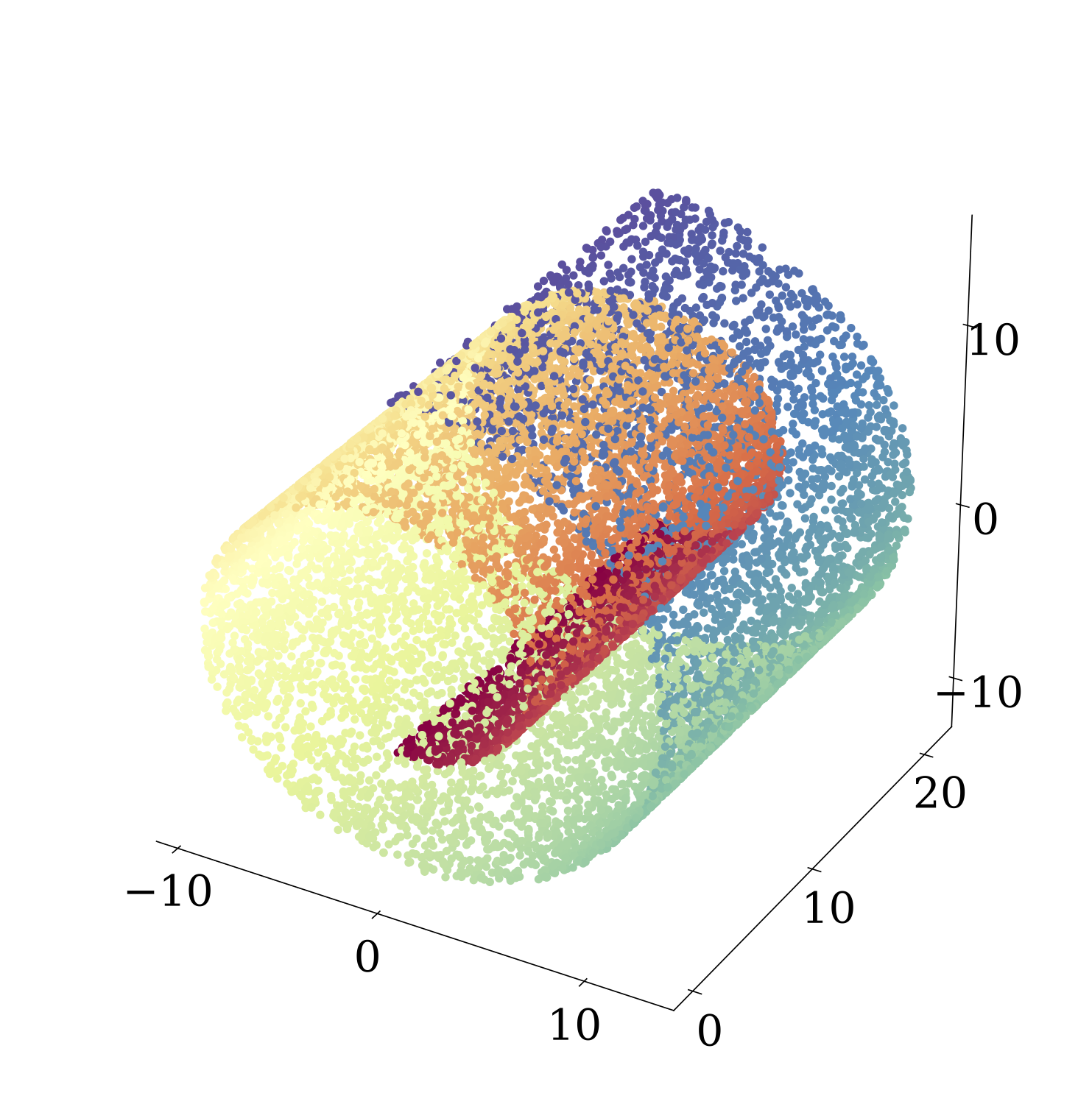}
  \caption{}
  \label{fig:roll1}
\end{subfigure}
\begin{subfigure}{.18\textwidth}
 \centering
 \includegraphics[width=\linewidth]{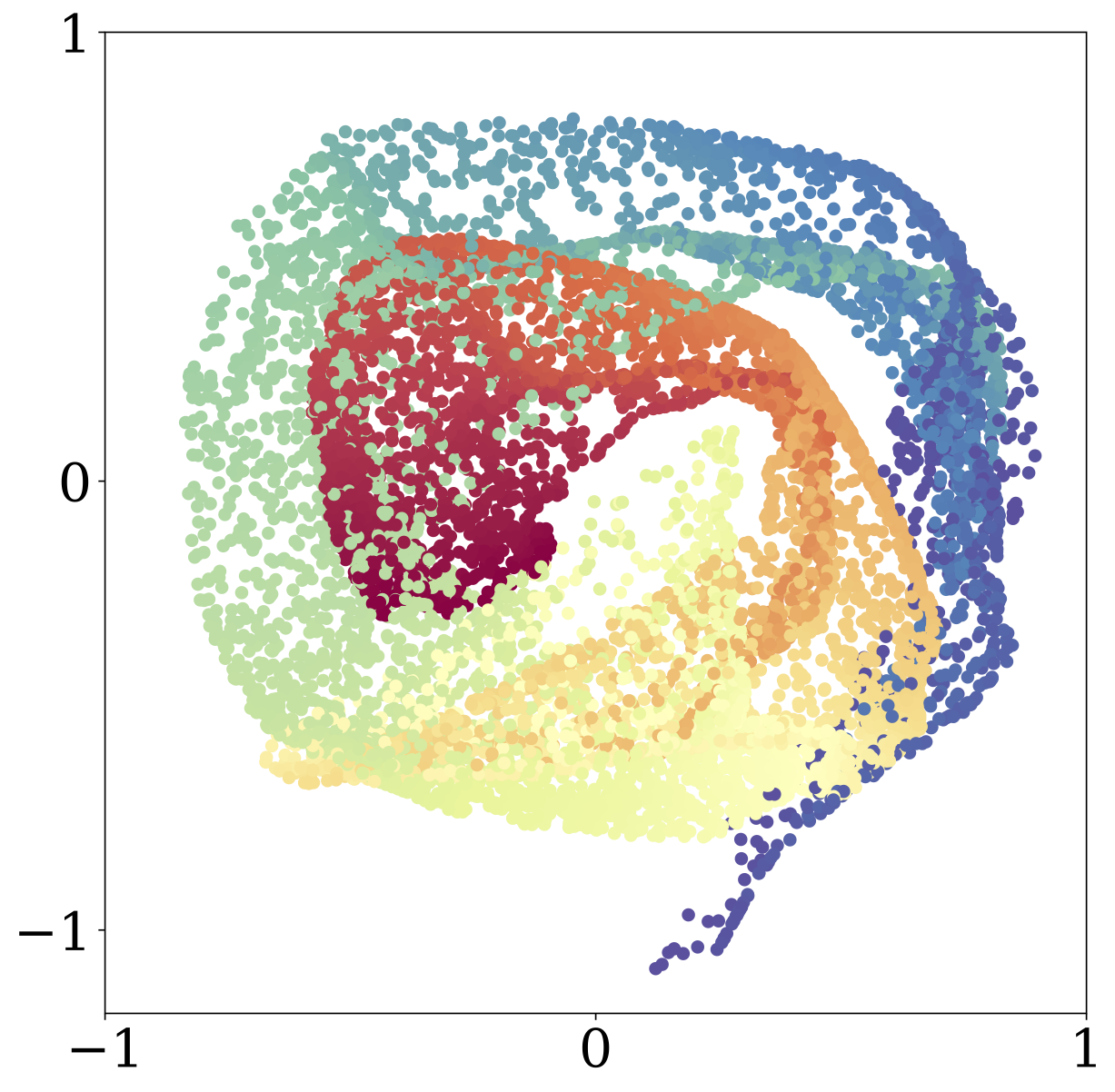}
 \caption{}
 \label{fig:roll2}
\end{subfigure}
\begin{subfigure}{.18\textwidth}
  \centering
  \includegraphics[width=\linewidth]{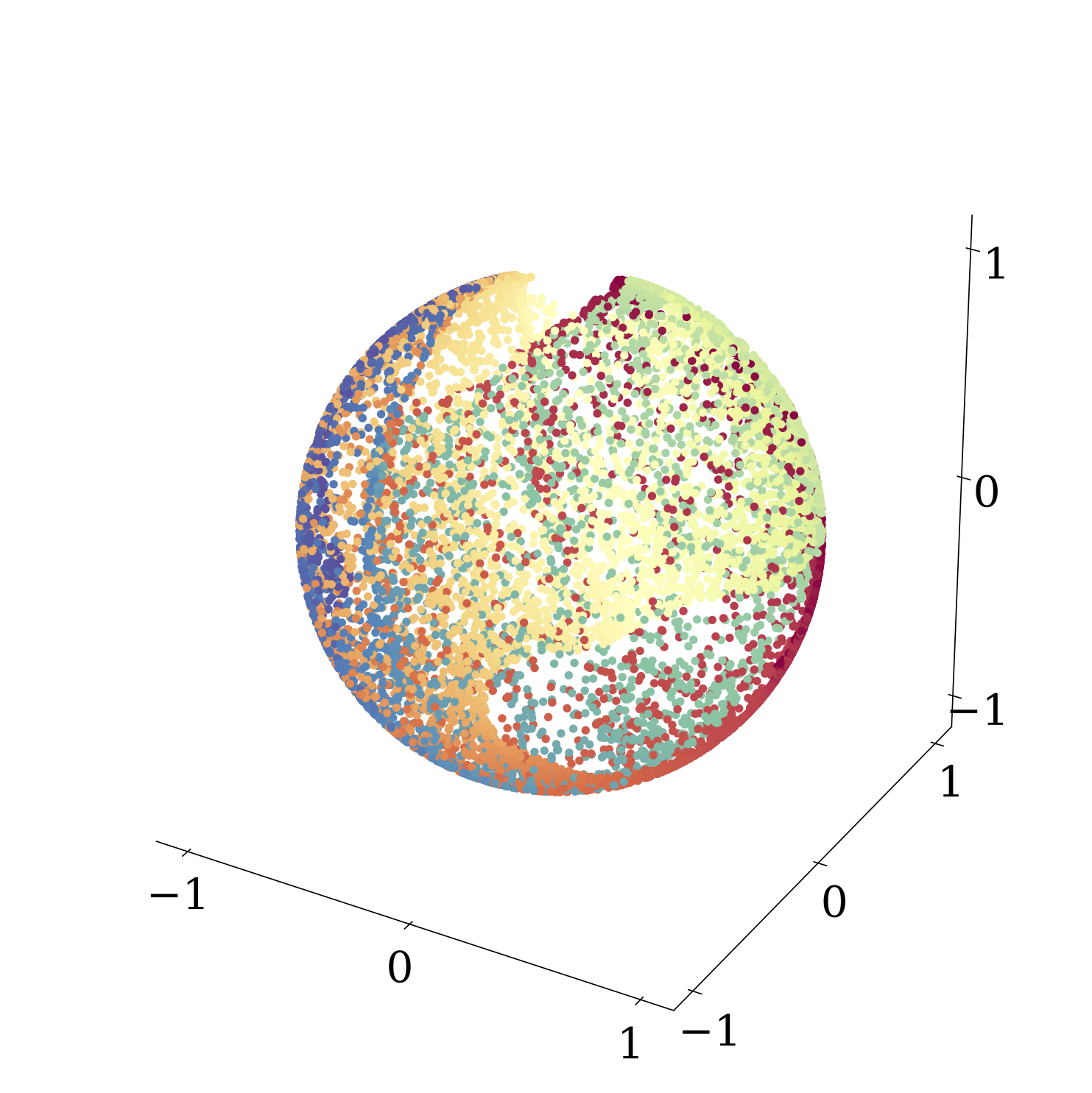}
  \caption{}
  \label{fig:roll3}
\end{subfigure}
\begin{subfigure}{.18\textwidth}
  \vspace{1pt}
  \centering
  \includegraphics[width=\linewidth]{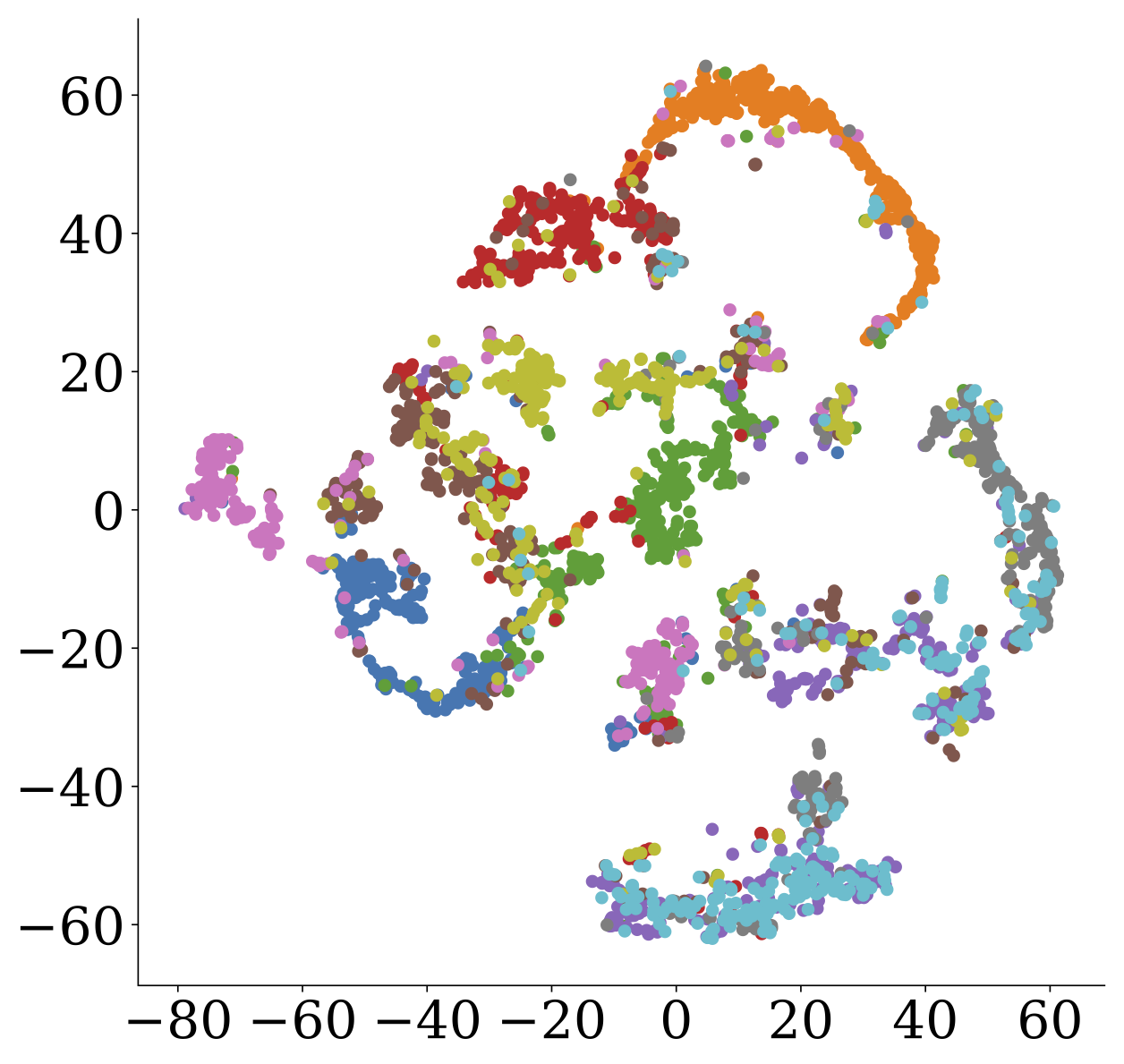}
  \caption{}
  \label{fig:roll4}
\end{subfigure}
 \begin{subfigure}{.18\textwidth}
   \vspace{1pt}
  \centering
  \includegraphics[width=\linewidth]{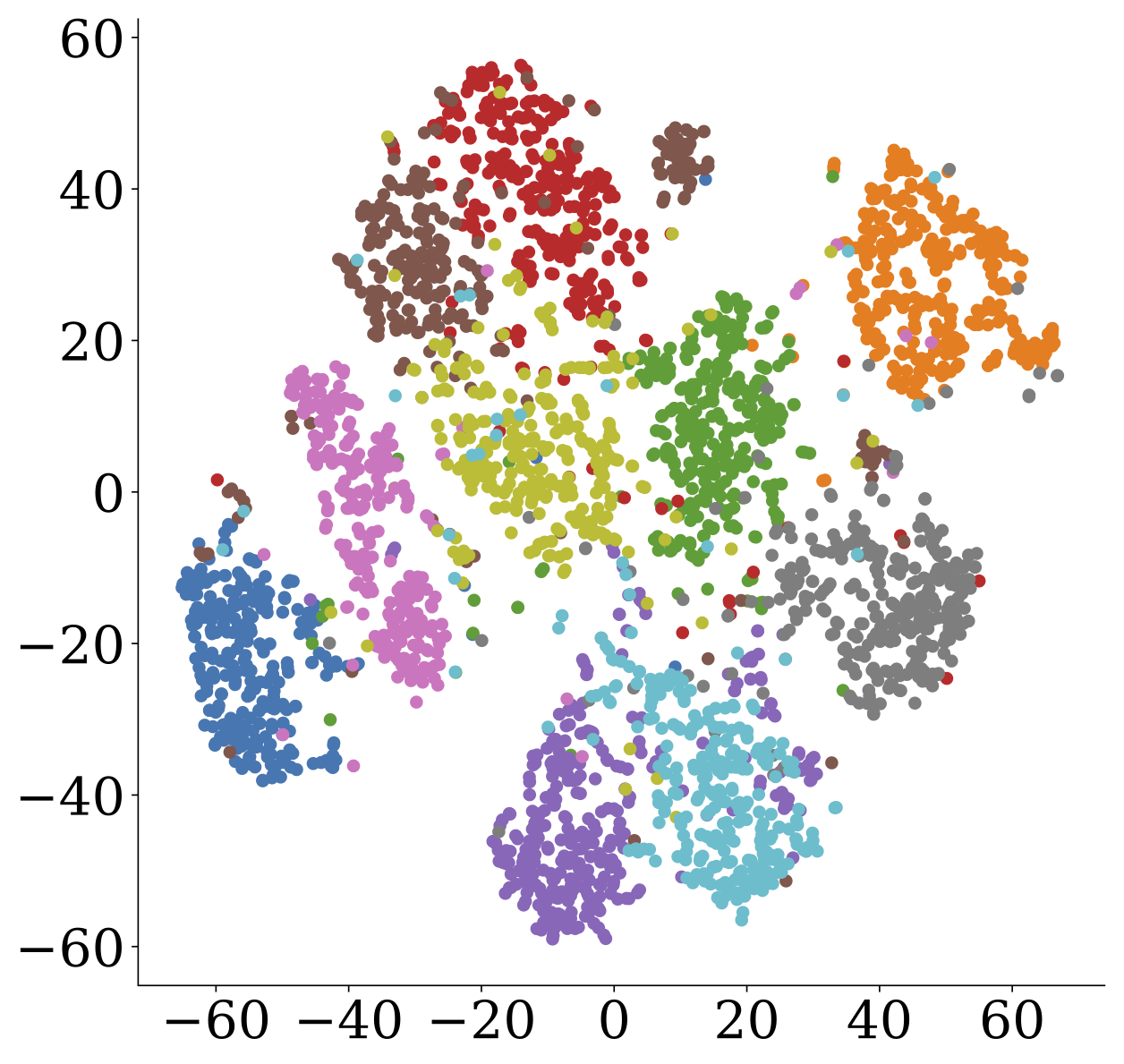}
  \caption{}
  \label{fig:simplex0}
 \end{subfigure}
 \vspace{-0.3cm}
\caption{a) Swiss Roll and its b) squared and c) spherical embeddings learned by Sinkhorn encoders. MNIST embedded onto a 10D sphere viewed through $t$-SNE, with classes by colours: d) encoder only or e) encoder + decoder.}
\vspace{-0.5cm}
\end{figure*}
\section{CLOSED FORM OF THE $2$-WASSERSTEIN DISTANCE}\label{sec:w2gae}

The $2$-Wasserstein distance $W_2(Q_Z,P_Z)$ has a closed form in Euclidean space if both $Q_Z$ and $P_Z$ are Gaussian 
(\cite{peyrecuturicomputational} Rem. 2.31):
\begin{align}
 W_2^2(&\mathcal{N}(\mu_1,\Sigma_1),\mathcal{N}(\mu_2,\Sigma_2)) = \|\mu_1-\mu_2\|_2^2 \nonumber\\
&+  \mathrm{tr}\left(\Sigma_1 + \Sigma_2 -2 \left(\Sigma_2^{\frac{1}{2}}\, \Sigma_1\, \Sigma_2^{\frac{1}{2}}  \right)^{\frac{1}{2}}   \right),
\end{align}
which will further simplify if $P_Z$ is standard Gaussian. Even though the aggregated posterior $Q_Z$ might not be Gaussian we use the above formula for matching and backpropagation, by estimating $ \mu_1$ and 
$\Sigma_1$ on minibatches of $Q_Z$ via the standard formulas:
$\hat \mu_1 := \frac{1}{M}\sum_{i=1}^M \tilde z_i$ and 
$\hat\Sigma_1: = \frac{1}{M-1} \sum_{i=1}^M ( \tilde z_i - \hat \mu_1) ( \tilde z_i - \hat \mu_1)^T $.
We refer to this method as W2GAE (Wasserstein Gaussian AutoEncoder).
We will compare this method against SAE and other baselines as discussed next in the related work section.

\section{RELATED WORK}

The Gaussian prior is common in VAE's for the reason of tractability. In fact, changing the prior and/or the approximate posterior distributions requires the use of tractable densities and the appropriate reparametrization trick. 
A hyperspherical prior is used by \cite{davidson2018hyperspherical} with improved experimental performance; the algorithm models a Von Mises-Fisher posterior, with a non-trivial posterior sampling procedure and a reparametrization trick based on rejection sampling. Our implicit encoder distribution sidesteps these difficulties. Recent advances on variable reparametrization can also simplify these requirements \citep{figurnov2018implicit}. 
We are not aware of methods embedding on probability simplices, except the use of Dirichlet priors by the same \cite{figurnov2018implicit}. 

\cite{hoffman2016elbo} showed that the objective of a VAE does not force the aggregated posterior and prior to match, and that the mutual information of input and codes may be minimized instead. Just like the WAE, SAE avoids this effect by construction. \cite{makhzani2015adversarial} and WAE improve latent matching by GAN/MMD. With the same goal, \cite{alemi2016deep} and \cite{tomczak2017vae} introduce learnable priors in the form of a mixture of posteriors, which can be used in SAE as well. 

The \cite{sinkhorn} algorithm gained interest after \cite{cuturi} showed its application for fast computation of Wasserstein distances. The algorithm has been applied to ranking \citep{adams2011ranking}, domain adaptation \citep{courty2014domain}, multi-label classification \citep{frogner}, metric learning \citep{huang2016supervised} and ecological inference \citep{muzellec2017tsallis}. \cite{santa2017deeppermnet, linderman2017reparameterizing} used it for supervised combinatorial losses. Our use of the Sinkhorn for generative modeling is akin to that of \cite{geneway2017learning}, which matches data and model samples with adversarial training, and to \cite{ambrogioni2018wasserstein}, which matches samples from the model joint distribution and a variational joint approximation.
WAE and WGAN objectives are linked respectively to primal and dual formulations of OT \citep{tolstikhin2018wasserstein}.

Our approach for training the encoder alone qualifies as self-supervised representation learning \citep{donahue2016adversarial, noroozi2016unsupervised, noroozi2017representation}. 
As in noise-as-target (NAT) \citep{bojanowski2017unsupervised} and in contrast to most other methods, we can sample pseudo labels (from the prior) independently from the input. In Appendix \ref{app:nat-compare} we show a formal connection with NAT.

Another way of estimating the $2$-Wasserstein distance in Euclidean space is the Sliced Wasserstein AutoEncoder (SWAE) \citep{kolouri2018sliced}. The main idea is to sample one-dimensional lines in Euclidean space and exploit the explicit form of the $2$-Wasserstein distance in terms of cumulative distribution functions in the one-dimensional setting. We will compare our methods to SWAE as well.
\begin{table*}[ht]
\vspace{-12pt}
\small
            \centering
            \begin{tabular}{lll|rrrr|rrrr}         
            \multicolumn{3}{c}{} & \multicolumn{4}{c} {MNIST} & \multicolumn{4}{c} {CelebA} \\ \hline 
              \textbf{method} & \textbf{prior} & \textbf{cost} & 
              \textbf{$\beta$} & \textbf{MMD}  & \textbf{RE} & \textbf{FID} &
              \textbf{$\beta$} & \textbf{MMD} & \textbf{RE} & \textbf{FID}\\ \hline
                VAE & $\mathcal N$ & KL & 1 
                 & 0.28 & 12.22 & {\bf 11.4}
                 & 1 &  0.20& 94.19 & {\bf55}\\
                $\beta$-VAE & $\mathcal N$  & KL & 0.1 
                 & 2.20 & 11.76 & 50.0 
                 & 0.1 &  0.21 & 67.80 & 65\\
                WAE & $\mathcal N$  & MMD & 100 
                 & 0.50 & 7.07 & 24.4 
                 &$2000^*$ & 0.21 & 65.45 & {\bf 58} \\
                SWAE & $\mathcal N$  & SW & 100
                 & 0.32 & 7.46 & 18.8
                 & 100 & 0.21&65.28 & 64\\
                W2GAE (ours) & $\mathcal N$ & $W^2_2$ 
                  & 1     &  0.67 &7.04 & 30.5
                 & 1 &  0.20 & 65.55& {\bf 58} \\
                 HAE (ours) & $\mathcal N$ & Hungarian & 100
                 & 5.79 & 11.84 & 16.8 
                 & 100 &32.09 & 84.51 & 293 \\
                SAE (ours) & $\mathcal N$ & Sinkhorn & 100
                 & 5.34 & 12.81 & 17.2
                 & 100 & 4.82&90.54& 187 \\
                \hline 
                HVAE$^\dagger$ &$\mathcal H$  & KL & 1 
                 & 0.25 & 12.73 & 21.5 
                 & - & - &- & - \\
                WAE & $\mathcal H$  & MMD & 100 
                 & 0.24 & 7.88 & 22.3
                 & $2000^*$  & 0.25& 66.54 & 59 \\
                SWAE & $\mathcal H$ & SW & 100
                 & 0.24 & 7.80 & 27.6
                 & 100& 0.41 & 63.64& 80 \\
                 HAE (ours) & $\mathcal H$ & Hungarian & 100
                 & 0.23 & 8.69 & {\bf 12.0}
                 &100& 0.26 & 63.49 & {\bf 58} \\
                SAE (ours) & $\mathcal H$& Sinkhorn & 100
                 & 0.25 & 8.59 & {\bf 12.5}
                &100 &  0.24 & 63.97 & {\bf 56} \\
            \end{tabular}
                \vspace{-0.2cm}
            \caption{Results of the autoencoding task. Top 3 results for the FID scores are indicated with boldface numbers. 
            We compute MMD in latent space to evaluate the matching between the aggregated posterior and prior. 
           MMD results are reported times $10^2$. Note that MMD scores are not comparable for different priors. 
            For SAE and the Gaussian prior, we used $\epsilon=10$ as lower values led to numerical instabilities. For the hypersphere we set $\epsilon=0.1$. *The value of $\beta=2000$ is similar to the value $\lambda=100$ as used in \citep{tolstikhin2018wasserstein}, as a prefactor of $0.05$ was used there for the reconstruction cost. $^\dagger$Comparing with \cite{davidson2018hyperspherical} in high-dimensional latent spaces turned out to be unfeasible, due to CPU-based computations.
            }
            \vspace{-0.5cm}
            \label{table:vae}
        \end{table*}
\section{EXPERIMENTS}
%
\subsection{REPRESENTATION LEARNING WITH SINKHORN ENCODERS}\label{exp:encoder}
We demonstrate qualitatively that the Sinkhorn distance is a valid objective for unsupervised feature learning by training the encoder in isolation. The task consists of embedding the input distribution in a lower dimensional space, while preserving the local data geometry and minimizing the loss function $L = \frac{1}{M} \langle R^* , \tilde C\rangle_F$, with $\tilde c(z,z') = \| z - z'\|_2^2$. Here $M$ is the minibatch size.

We display the representation of a 3D Swiss Roll and MNIST. 
For the Swiss Roll we set $\varepsilon = 10^{-3}$, while for MNIST it is set to $0.5$, and $L$ is picked to ensure convergence. 
For the Swiss roll (Figure \ref{fig:roll1}), we use a 50-50 fully connected network with ReLUs.
Figures \ref{fig:roll2}, \ref{fig:roll3} show that the local geometry of the Swiss Roll is conserved in the new representational spaces --- a square and a sphere. 
Figure \ref{fig:roll4} shows the $t$-SNE visualization \citep{maaten2008visualizing} of the learned representation of the MNIST test set. With neither labels nor reconstruction error, we learn an embedding that is aware of class-wise clusters.
Minimization of the Sinkhorn distance achieves this by encoding onto a $d$-dimensional hypersphere with a uniform prior, such that points are encouraged to map far apart.
A contractive force is present due to the inductive prior of neural networks, which are known to be Lipschitz functions. On the one hand, points in the latent space disperse in order to fill up the sphere; on the other hand, points close on image space cannot be mapped too far from each other. As a result, local distances are conserved while the overall distribution is spread.
When the encoder is combined with a decoder $G$ the contractive force is enlarged: they collaborate in learning a latent space which makes reconstruction possible despite finite capacity; 
see Figure \ref{fig:simplex0}.
\begin{figure*}[ht]
\vspace{-8pt}
\centering
\begin{subfigure}{.37\textwidth}
  \includegraphics[width=\linewidth]{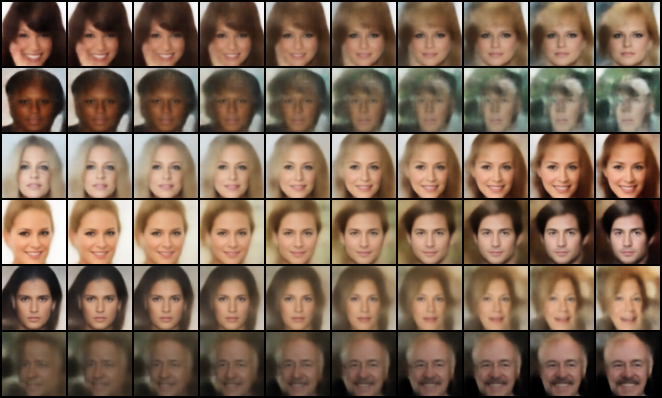}
\end{subfigure}
\begin{subfigure}{.37\textwidth}
  \includegraphics[width=\linewidth]{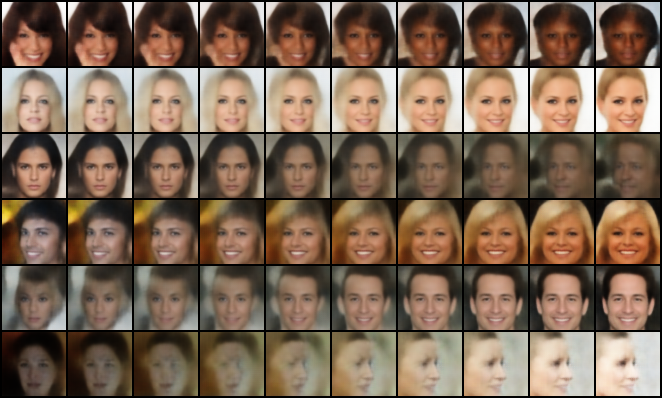}
\end{subfigure}
\begin{subfigure}{.22\textwidth}
  \includegraphics[width=\linewidth]{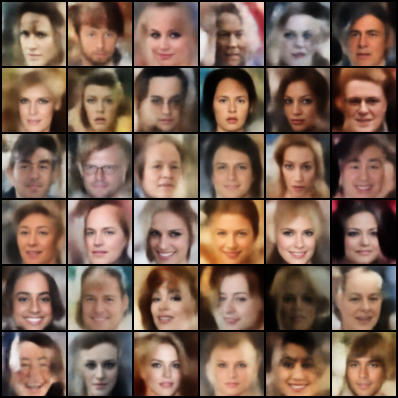}
\end{subfigure}
\begin{subfigure}{.37\textwidth}
  \includegraphics[width=\linewidth]{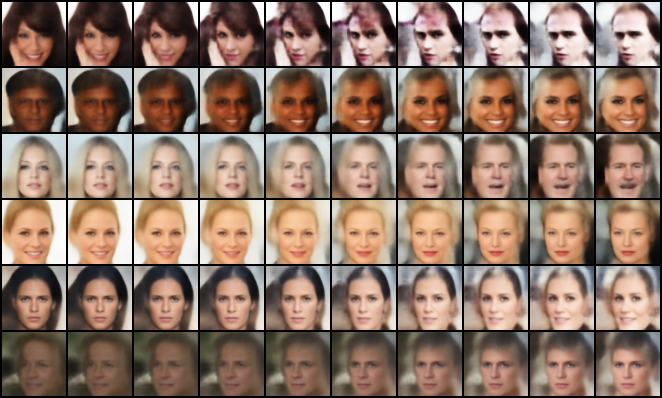}
\end{subfigure}
\begin{subfigure}{.37\textwidth}
  \includegraphics[width=\linewidth]{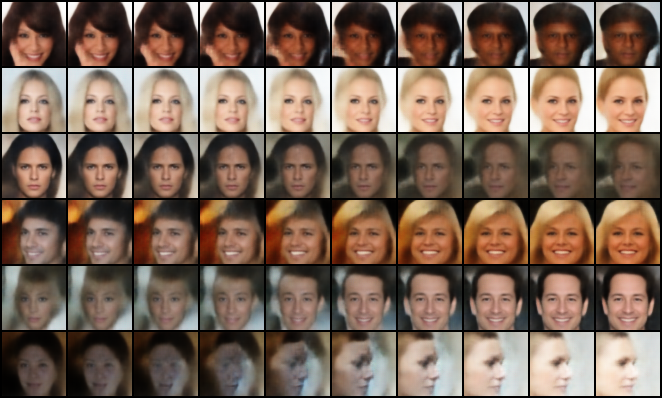}
\end{subfigure}
\begin{subfigure}{.22\textwidth}
  \includegraphics[width=\linewidth]{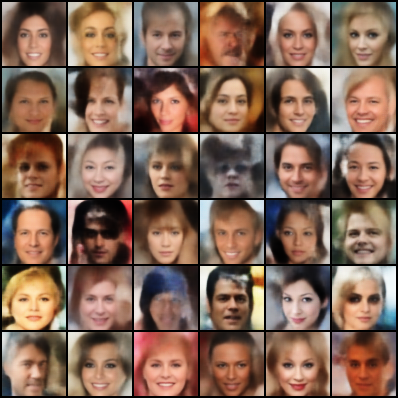}
\end{subfigure}
\vspace{-0.2cm}
\caption{From left to right: CelebA extrapolations, interpolations, and samples. Models from Table \ref{table:vae}: WAE with a Gaussian prior (top) and SAE with a uniform prior on the hypersphere (bottom).}
\label{fig:big1}
\vspace{-0.5cm}
\end{figure*}  
\begin{figure*}[!b] 
\vspace{-7pt}
 \centering
\begin{subfigure}{.21\textwidth}
  \centering
  \includegraphics[width=\linewidth]{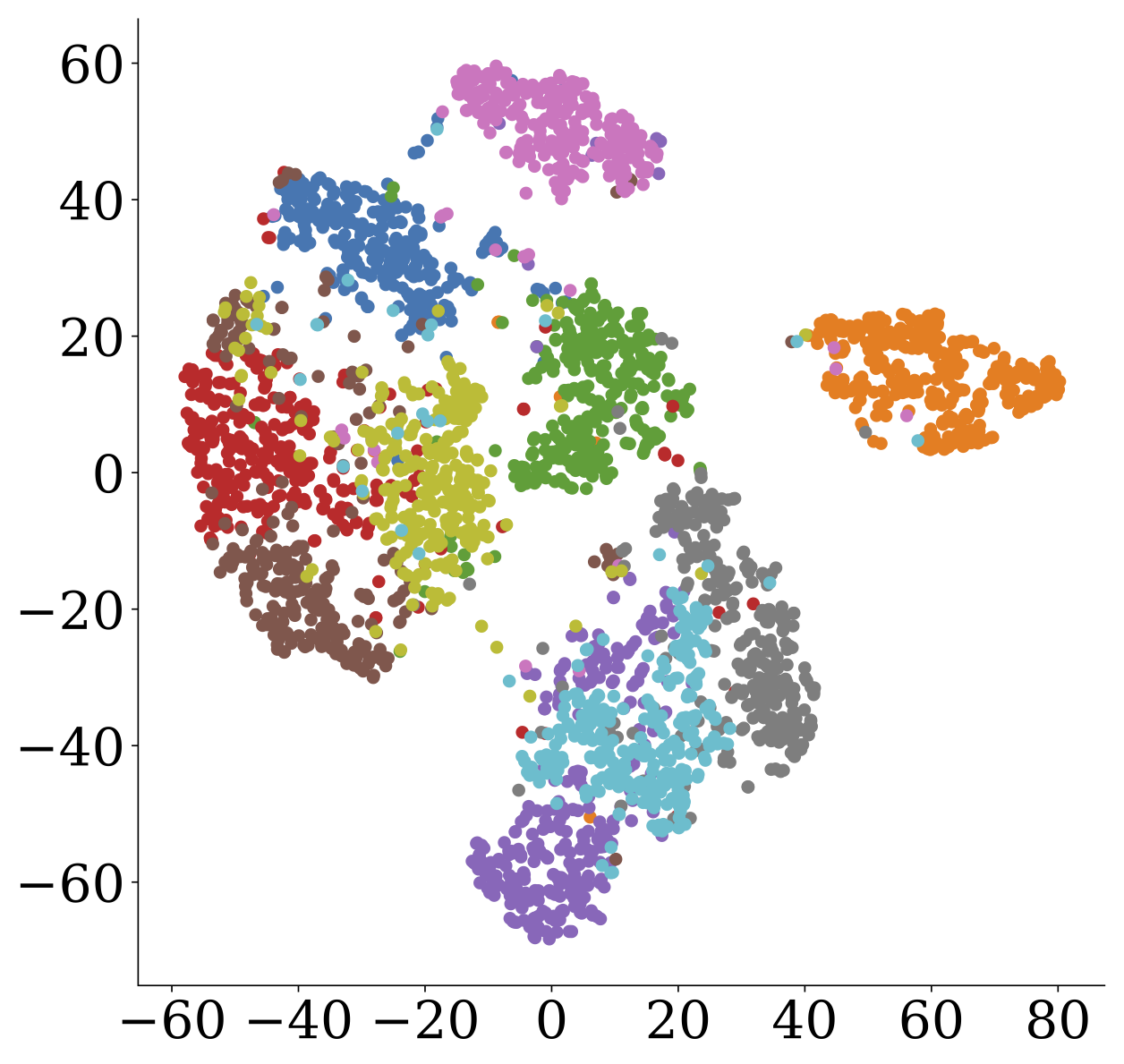}
  \caption{}
  \label{fig:simplex1}
\end{subfigure}
\hspace{1pt}
\begin{subfigure}{.21\textwidth}
 \centering
 \includegraphics[width=\linewidth]{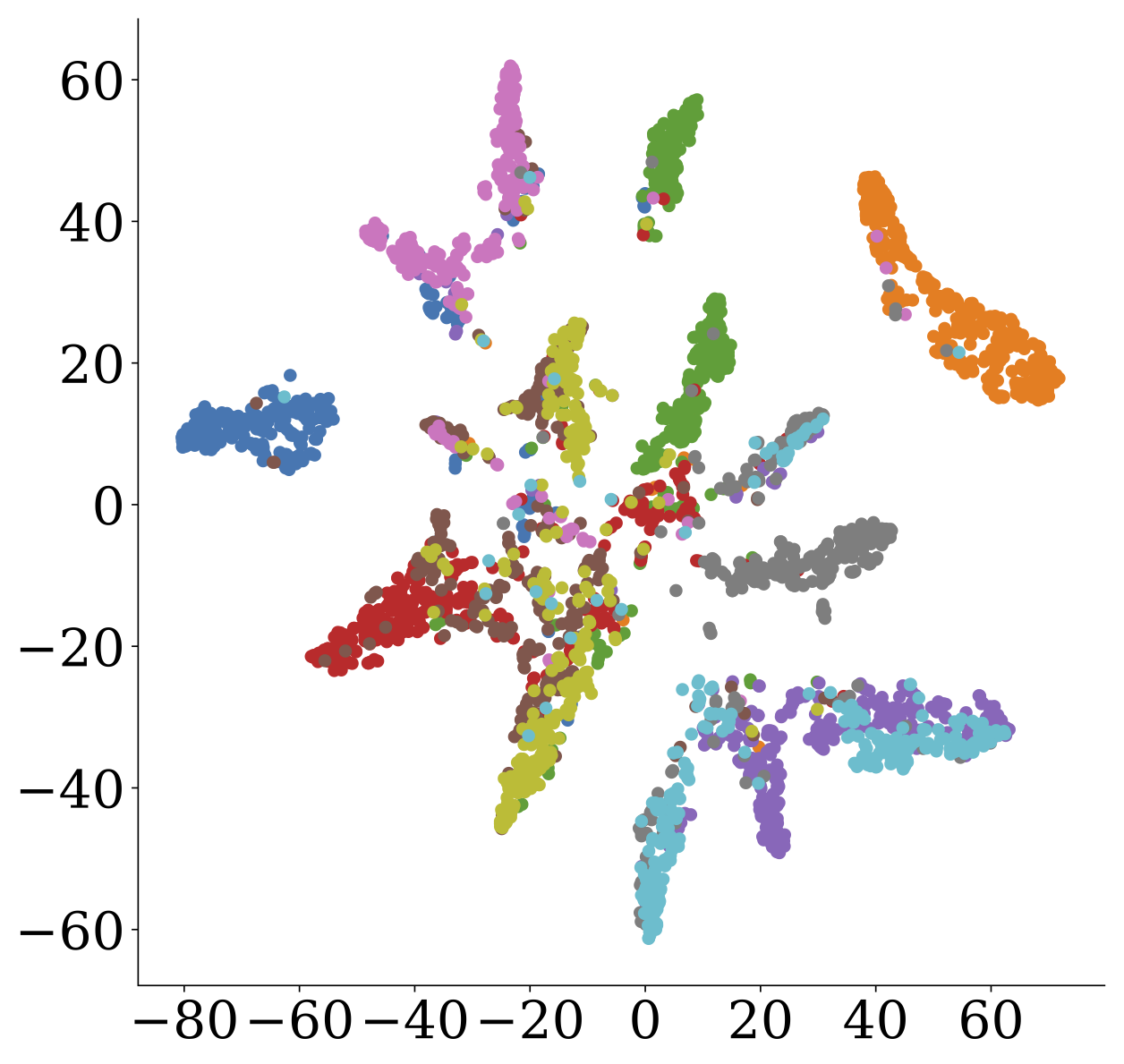}
 \caption{}
 \label{fig:simplex2}
\end{subfigure}
\hspace{1pt}
\begin{subfigure}{.22\textwidth}
 \centering
 \includegraphics[width=\linewidth]{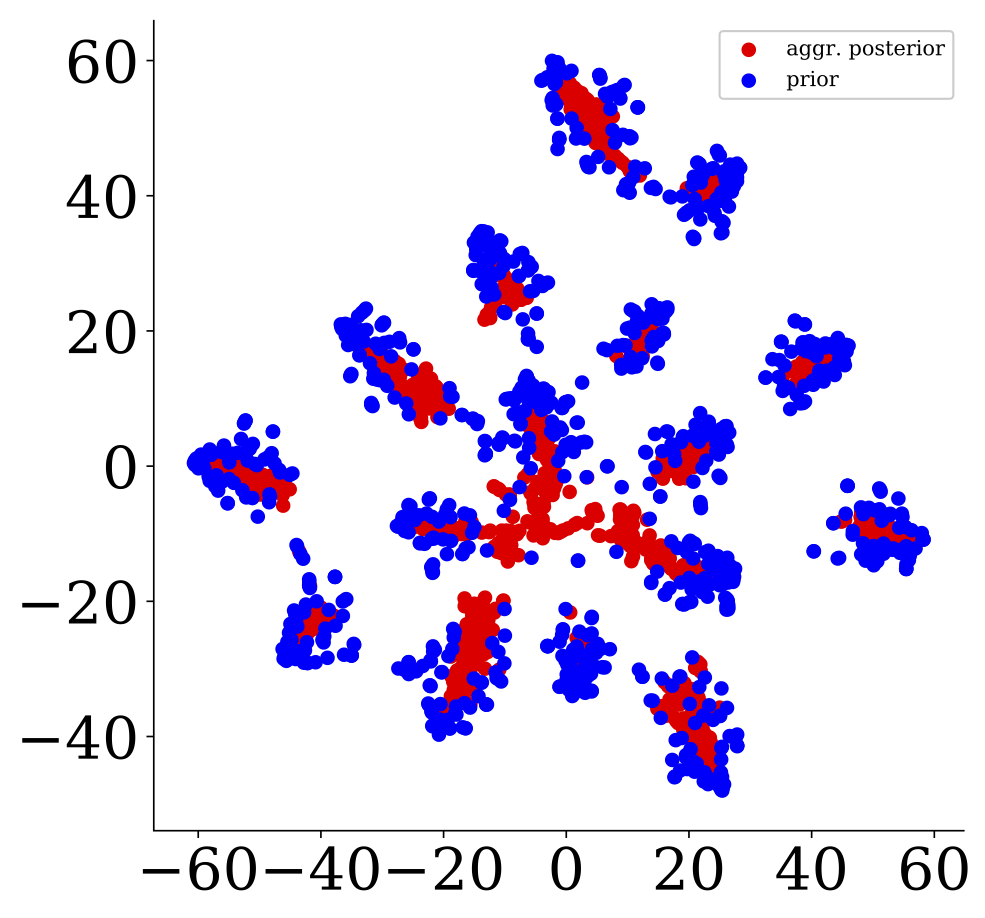}
 \caption{}
 \label{fig:simplex3}
\end{subfigure}
\hspace{4pt}
\begin{subfigure}{.25\textwidth}
  \centering
  \includegraphics[width=\linewidth]{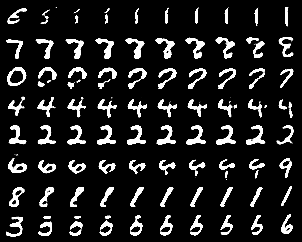}
  \caption{}
  \label{fig:simplex4}
\end{subfigure}
\begin{subfigure}{.0265\textwidth}
  \centering
  \includegraphics[width=\linewidth]{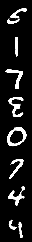}
  \caption{}
  \label{fig:simplex5}
\end{subfigure}
\vspace{-0.3cm}
\caption{$t$-SNEs of SAE latent spaces on MNIST: a) $10$-dim $\mathrm{Dir}(1/2)$ and b) $16$-dim $\mathrm{Dir}(1/5)$ priors. For the latter: c) aggregated posterior (red) vs. prior (blue), d) vertices interpolation and e) samples from the prior.}
\vspace{-0.6cm}
\end{figure*}

\subsection{AUTOENCODING EXPERIMENTS}\label{exp:generators}
      
For the autoencoding task we compare SAE against ($\beta$)-VAE, HVAE, SWAE and WAE-MMD. We furthermore denote the model that matches the samples in latent space with the Hungarian algorithm with HAE. Where compatible, all methods are evaluated both on the hypersphere and with a standard normal prior. Results from our proposed W2GAE method as discussed in section \ref{sec:w2gae} for Gaussian priors are shown as well. We compute FID scores \citep{heusel2017gans} on CelebA and MNIST. For MNIST we use LeNet as proposed in \citep{binkowski2018demystifying}.
For details on the experimental setup, see Appendix \ref{app:architectures}.
The results for MNIST and CelebA are shown in Table \ref{table:vae}.  Extrapolations, interpolations and samples of WAE and SAE for CelebA are shown in Fig. \ref{fig:big1}. Visualizations for MNIST are shown in Appendix \ref{sec:visualizations}. Interpolations on the hypersphere are defined on geodesics connecting points on the hypersphere. FID scores of SAE with a hyperspherical prior are on par or better than the competing methods. Note that although the FID scores for the VAE are slightly better than that of SAE/HAE, the reconstruction error of the VAE is significantly higher. Surprisingly, the simple W2GAE method is on par with WAE on CelebA.

For the Gaussian prior on CelebA, both HAE and SAE perform very poorly. In appendix \ref{app:highd} we analyzed the behaviour of the Hungarian algorithm in isolation for two sets of samples from high-dimensional Gaussian distributions. 
The Hungarian algorithm finds a better matching between samples from a smaller variance Gaussian with samples from the standard normal distribution. This behaviour gets worse for higher dimensions, and also occurs for the Sinkhorn algorithm.
This might be due to the fact that most probability mass of a high-dimensional isotropic Gaussian with standard deviation $\sigma$ lies on a thin annulus at radius $\sigma \sqrt{d}$ from its origin. For a finite number of samples the $L_2^2$ cost function can lead to a lower matching cost for samples between two annuli of different radii. This effect leads to an encoder with a variance lower than one. When sampling from the prior after training, this yields saturated sampled images. See Appendix \ref{sec:visualizations} for reconstructions and samples for HAE with a Gaussian prior on CelebA. Note that neither SWAE and W2GAE suffer from this problem in our experiments, even though these methods also provide an estimate of the 2-Wasserstein distance. For W2GAE this problem does start at even higher dimensions (Appendix \ref{app:highd}).


\subsection{DIRICHLET PRIORS}\label{exp:prior-decouple}

We further demonstrate the flexibility of SAE by using Dirichlet priors on MNIST. The prior draws samples on the probability simplex; hence we constrain the encoder by a final softmax layer. We use priors that concentrate on the vertices with the purpose of clustering the digits. A $10$-dimensional $\mathrm{Dir}(1/2)$ prior (Figure \ref{fig:simplex1}) results in an embedding qualitatively similar to the uniform sphere (\ref{fig:simplex0}).
With a more skewed prior $\mathrm{Dir}(1/5)$, the latent space could be organized such that each digit is mapped to a vertex, with little mass in the center. We found that in dimension $10$ this is seldom the case, as multiple vertices can be taken by the same digit to model different styles, while other digits share the same vertex.
 We therefore experiment with a $16$-dimensional $\mathrm{Dir}(1/5)$, which yields more disconnected clusters (\ref{fig:simplex2}); the effect is evident when showing the prior and the aggregated posterior that tries to cover it (\ref{fig:simplex3}). 
Figure \ref{fig:simplex4} (leftmost and rightmost columns) shows that every digit $0-9$ is indeed represented on one of the 16 vertices, while some digits are present with multiple styles, e.g. the $7$.
The central samples in the Figure are the interpolations obtained by sampling on edges connecting vertices -- no real data is autoencoded.
Samples from the vertices appear much crisper than other prior samples (\ref{fig:simplex5}), a sign of mismatch between prior and aggregated posterior on areas with lower probability mass.
Finally, we could even learn the Dirichlet hyperparameter(s) with a reparametrization trick \citep{figurnov2018implicit} and let the data inform the model on the best prior.

\section{CONCLUSION}

We introduced a generative model built on the principles of Optimal Transport. Working with empirical Wasserstein distances and deterministic networks provides us with a flexible likelihood-free framework for latent variable modeling. 


\newpage

\bibliography{references}

\bibliographystyle{apalike}

\clearpage


\appendix 

\section{RESTRICTING TO DETERMINISTIC AUTOENCODERS}\label{proof:mainth1}

We first improve the characterization of Equation \ref{main}, which is formulated in terms of stochastic encoders $Q(Z|X)$ and deterministic decoders $G(X|Z)$. In fact, it is possible to restrict the learning class to that of deterministic encoders and deterministic decoders and thus to fully deterministic \emph{autoencoders}:
\begin{thm}\label{mainth}
Let $P_X$ be not atomic and $G(X|Z)$ deterministic. Then for every continuous cost $c$:
\begin{equation*}
W_c(P_X,P_G) = \inf_{\substack{Q~\text{det. }:\\Q_Z = P_Z}} \E_{X \sim P_X}[c(X,G(Q(X)))].
\end{equation*} 
Using the cost $c(x,y) = \|x-y\|_p^p$, the equation holds with $W_p^p(P_X,P_G)$ in place of $W_c(P_X,P_G)$.
\end{thm}
The statement is a direct consequence of the equivalence between the Kantorovich and Monge formulations of OT \citep{villani2008optimal}. We remark that this result is stronger than, and can be used to deduce Equation \ref{main}; see \ref{proof:consequence} for a proof.

The basic tool to prove Theorem \ref{mainth} is the equivalence between Monge and Kantorovich formulation of optimal transport. For convenience we formulate its statement and we refer to \cite{villani2008optimal} for a more detailed explanation. 

\begin{thm}[Monge-Kontorovich equivalence]\label{monge}
Given $P_X$ and $P_Y$ probability distributions on $\mathcal{X}$  such that $P_X$ is not atomic, $c:\mathcal{X} \times \mathcal{X} \rightarrow \R$ continuous, we have
\begin{equation}
W_c(P_X,P_Y) = \inf_{\substack{T: \mathcal{X} \rightarrow \mathcal{X}: \\ T_\# P_X = P_Y}} \int_{\mathcal{X}} c(x,T(x)) \, dP_X(x).
\end{equation}

\end{thm}

We are now in position to prove Theorem \ref{mainth}. We will prove it for a general continuous cost $c$.

\emph{Proof of Theorem \ref{mainth}}.
Notice that as the encoder $Q(Z|X)$ is deterministic there exists $Q:\mathcal{X} \rightarrow \mathcal{Z}$ such that $Q_Z = Q_\# P_X$ and $Q(Z|X) = \delta_{\{Q(x) = z\}}$. Hence
\begin{align*}
\E_{X \sim P_X} & \E_{Z \sim Q(Z|X)}[c(X,G(Z))] \\
& =  \int_{\mathcal{X} \times \mathcal{Z}} c(x,G(z)) \, dP_X(x) d\delta_{\{Q(x) = z\}}(z) \\
&= \int_{\mathcal{X}} \, dP_X(x) \int_{\mathcal{Z}} c(x,G(z)) d\delta_{\{Q(x) = z\}}(z)\\
&= \int_{\mathcal{X}}  c(x,G(Q(x))) \, dP_X(x)\,.
\end{align*}
Therefore
\begin{align*}
\inf_{\substack{Q \text{ det. }:\\Q_Z = P_Z}}& \E_{X \sim P_X}\E_{Z \sim Q(Z|X)}[c(X,G(Z))] \\
& = \inf_{\substack{Q:\mathcal{X}\rightarrow \mathcal{Z}\\ Q_Z = P_Z}}\int_{\mathcal{X}}  c(x,G(Q(x))) \, dP_X\,.
\end{align*}

We now want to prove that
\begin{align}\label{claim}
\{G \circ Q : Q_\#P_X = P_Z\} = \{T: \mathcal{X} \rightarrow \mathcal{X} : T_\#P_X = P_G\}\,.
\end{align}

For the first inclusion $\subset$ notice that for every $Q : \mathcal{X} \rightarrow \mathcal{Z}$ such that $Q_Z = P_Z$ we have that $G\circ Q : \mathcal{X} \rightarrow\mathcal{X}$ and 
\begin{align*}
(G\circ Q)_{\#}P_X = G_\# Q_\# P_X = G_\# P_Z \,.
\end{align*}
For the other inclusion $\supset$ consider $T: \mathcal{X} \rightarrow \mathcal{X}$ such that $T_\# P_X = P_G= G_\# P_Z$. We want first to prove that there exists a set $A \subset \mathcal{X}$ with $P_X(A) = 1$  such that $G: \mathcal{Z} \rightarrow  T(A)$ is surjective. Indeed if it does not hold there exists $B \subset \mathcal{X}$ with $P_X(B) > 0$ and $G^{-1}(T(B)) = \emptyset$. Hence
\begin{align*}
0 = G_\# P_Z(T(B)) = T_\# P_X(T(B)) = P_X(B) > 0
\end{align*}
that is a contraddiction.
Therefore by standard set theory the map $G: \mathcal{Z} \rightarrow  T(A)$ has a right inverse that we denote by $\widetilde G$.
Then define $Q=\widetilde G \circ T$. Notice that $G\circ Q = G\circ \widetilde G \circ T = T$ almost surely in $P_X$ and also
\begin{align*}
(\widetilde G \circ T)_\#P_X = P_Z\,. 
\end{align*}
Indeed for any $A \subset \mathcal{Z}$ Borel we have
\begin{align*}
(\widetilde G \circ T)_\# P_X (A) & =  (\widetilde G \circ G)_\# P_Z (A) \\
& = P_Z(\widetilde G^{-1}(G^{-1}(A)) \\
& = P_Z(A)\,. 
\end{align*}
This concludes the proof of the claim in (\ref{claim}). Now we have
\begin{align*}
\inf_{\substack{Q:\mathcal{X}\rightarrow \mathcal{Z}\\ Q_\#(P_X) = P_Z}}& \int_{\mathcal{X}}  c(x,G(Q(x))) \, dP_X(x)= \\
& = \inf_{\substack{T:\mathcal{X}\rightarrow \mathcal{X}\\ T_\#(P_X) = P_G}}\int_{\mathcal{X}}  c(x,T(x)) \, dP_X(x)\,. 
\end{align*}
Notice that this is exactly the Monge formulation of optimal transport. Therefore by Theorem \ref{monge} we conclude that
\begin{align*}
\inf_{\substack{ Q \text{ det.} :\\ Q_Z = P_Z}} & \E_{X \sim P_X}\E_{Z \sim Q(Z|X)}[c(X,G(Z))]= \\
& = \inf_{ \Gamma \in \Pi(P_X,P_G)} \E_{(X,Y) \sim \Gamma}[ c(X,Y) ]
\end{align*}
as we aimed. \qed

\section{WAE AS A CONSEQUENCE}\label{proof:consequence}

The following proof will show that we get the WAE objective from equation \ref{main} \citep{tolstikhin2018wasserstein} as a consequence from \ref{mainth}.

\begin{proof}
Thanks to Theorem \ref{mainth} we can prove easily Equation \ref{main}. Indeed:
\begin{align*}
 &\qquad W_c(P_X,P_G)  \\
 =  & \inf_{\substack{Q \text{ det. }:\\ Q_Z = P_Z}} \E_{X \sim P_X}\E_{Z \sim Q(Z|X)}[c(X,G(Z))]\\
 \geq & \inf_{\substack{Q(Z|X):\\ Q_Z = P_Z}} \E_{X \sim P_X}\E_{Z \sim Q(Z|X)}[c(X,G(Z))]\,.
\end{align*}
For the opposite inequality given $Q(Z|X)$ such that 
$P_Z = \int Q(Z|X) dP_X$ define $Q(X,Y) = P_X \otimes [G_\#Q(Z|X)]$. It is a distribution on $\mathcal{X} \times \mathcal{X}$ and it is easy to check that $\pi^1_{\#} Q(X,Y) = P_X$ and $\pi^2_{\#} Q(X,Y) = G_\#P_Z$, where $\pi^1$ and $\pi^2$ are the projection on the first and the second component. Therefore
\begin{align*}
\{Q(Z|X) \mbox{ such that } Q_Z = P_Z\} \subset   \Pi(P_X,P_G) 
\end{align*}
and so
\begin{align*}
W_c(P_X,P_G) & \leq \inf_{Q(Z|X): Q_Z = P_Z} \int_{\mathcal{X} \times \mathcal{X}} c(x,y) \, dQ(x,y)\\
& = \int_{\mathcal{X}} \left[\int_{\mathcal{X}} c(x,y) \, d G_\# Q(Z|X)(y)\right] \, dP_X\\
 & =  \int_{\mathcal{X}} \left[\int_{\mathcal{X}} c(x,G(z)) \, d Q(Z|X)(z)\right] \, dP_X\,.
\end{align*}
\end{proof}

\section{PROOF OF THE MAIN THEOREM}\label{proof:propsuff}

In this section we finally use Theorem \ref{proof:mainth1} to prove Theorem \ref{holy-grail-app}. In what follows $d(\cdot,\cdot)$ denotes a distance. 
As a preliminary Lemma, we prove a Lipschitz property for the Wasserstein distance $W_p$. 

\begin{lemma}\label{wasser}
For every $P_X,P_Y$ distributions on a sample space $\mathcal{S}$  and a Lipschitz map $F$ (with respect to $d$) we have that
\begin{displaymath}
W_p(F_\#P_X,F_\#P_Y) \leq \gamma \cdot W_p(P_X,P_Y)\,,
\end{displaymath}
where $\gamma$ is the Lipschitz constant of $F$.
\end{lemma}

\begin{proof}
Recall that 
\begin{align*}
& W_p(F_\#P_X,F_\#P_Y)^p  \\
& = \inf_{\Gamma \in \Pi(F_\#P_X,F_\#P_Y)} \int_{\mathcal{S} \times \mathcal{S}} d(x,y)^p\, d\Gamma(x,y).
\end{align*}
Notice then that for every $\Gamma \in \Pi(P_X,P_Y)$ we have that $(F \times F)_\# \Gamma  \in \Pi(F_\#P_X, F_\# P_Y)$. Hence
\begin{align}\label{cont}
\{(F \times F)_\# \Gamma : \Gamma \in \Pi(P_X,P_Y)\} \subset  \Pi(F_\#P_X,F_\#P_Y).
\end{align}
From (\ref{cont}) we deduce that
\begin{align*}
& W_p  (F_\#P_X,F_\#P_Y)^p  \\
& \leq \inf_{\Gamma \in \Pi(P_X,P_Y)} \int_{\mathcal{S}\times \mathcal{S}} d(x,y)^p\, d(F\times F)_\#\Gamma \\
& = \inf_{\Gamma \in \Pi(P_X,P_Y)} \int_{\mathcal{S} \times \mathcal{S}} d(F(x),F(y))^p\, d\Gamma \\
& \leq \gamma^p \cdot (W_p(P_X,P_Y))^p\,.
\end{align*}
Taking the $p$-root on both sides we conclude.
\end{proof}

\emph{Proof of Theorem \ref{holy-grail-app}}. 

Let $G$ be a deterministic encoder. Using the triangle inequality of the Wasserstein distance and Lemma \ref{wasser} we obtain 
\begin{align}
W_p(P_X, G_\#P_Z) & \leq W_p(P_X,G_\# Q_Z) + W_p(G_\# Q_Z ,G_\#P_Z)  \nonumber\\
& \leq W_p(P_X, G_\# Q_Z) + \gamma \cdot W_p(Q_Z,P_Z)
\label{applicationwasser} 
\end{align}
for every $Q$ deterministic decoder. \\
As $(\mathrm{id}_{\mathcal{X}},G)_\#Q(Z|X)P_X \in \Pi(P_X,G_\#Q_Z)$ we deduce the following estimate:
\begin{equation}\label{2estimate}
 W_p(P_X, G_\#Q_Z) \le \sqrt[p]{\E_{X \sim P_X}[d(X,G(Q(X)))^p]}
\end{equation}
for every $Q$ deterministic decoder.
Now combinining Estimates \ref{applicationwasser} and \ref{2estimate} and using Theorem \ref{mainth} we obtain
\begin{align}
W_p(P_X, P_G) &\stackrel{(\ref{applicationwasser}), (\ref{2estimate})}{\leq} \inf_{Q \text{ det.}}  \sqrt[p]{\E_{X \sim P_X}\left[d(X,G(Q(X)))^p\right]} \nonumber\\
& \qquad   + \gamma \cdot W_p(Q_Z,P_Z) \nonumber\\
&\stackrel{~~~~~}{\leq} \inf_{\substack{Q \text{ det.}\\, Q_Z=P_Z}}  \sqrt[p]{\E_{X \sim P_X}[d(X,G(Q(X)))^p]}  \nonumber \\
 & \qquad + \gamma \cdot \underbrace{W_p(Q_Z,P_Z)}_{=0} \label{step10} \\
& \stackrel{Th.\ref{mainth}}{=} W_p(P_X, P_G) \nonumber. 
\end{align} 
Inequality in Step \ref{step10} holds because we restrict the domain of the $\textit{infimum}$, which in turns implies $W_p(Q_Z, P_Z)=0$. As a consequence all the inequalities are equalities and in particular \begin{align}
W_p(P_X,P_G) = & \inf_{Q \text{ det.}}\sqrt[p]{\E_{X \sim P_X}[d(X,G(Q(X)))^p]} \nonumber\\
& \qquad + \gamma \cdot W_p(Q_Z,P_Z). \label{rhs-corr}
\end{align}

Let $\mathcal{F}$ be any class of probabilistic encoders that at least contains a class of universal approximators.
This means that for every deterministic encoder $Q$ and for every $\varepsilon>0$ there exists $Q^\varepsilon \in \mathcal{F}$ that approximates $Q$ up to an error of $\varepsilon$ in the $L^p$-metric, namely:
\begin{equation}\label{lpapprox}
\sqrt[p]{\E_{X \sim P_X}[d(Q(X),Q^\varepsilon(X))^p]}\leq \varepsilon.
\end{equation}

Let $Q^*$ be an optimal measurable deterministic encoder that optimizes the right-hand side of Equation \ref{rhs-corr} among measurable deterministic encoder (or at least $\delta\le\gamma \varepsilon$ close to it) and 
$Q^{\varepsilon}$ an approximation of $Q^*$ as in Formula \ref{lpapprox}. 
Then with help of the triangle inequality for the $L^p$-metric and the Wasserstein distance:

\begin{align*}
 & \sqrt[p]{\E_{X \sim P_X}\left[d(X, G(Q^{\varepsilon}(X)))^p\right]}  + \gamma \cdot W_p(Q_Z^{\varepsilon},P_Z)\\
& \stackrel{\text{inequalities}}{\leq} \sqrt[p]{\E_{X \sim P_X}\left[d(X,G(Q^*(X)))^p\right]} \\
&\qquad + \underbrace{\sqrt[p]{\E_{X \sim P_X}\left[d(G(Q^*(X)), G(Q^{\varepsilon}(X)))^p\right]}}_{\le \gamma \cdot \varepsilon}  \\
&\qquad  +\gamma \cdot \underbrace{W_p(Q_Z^{\varepsilon},Q_Z^*)}_{\leq \varepsilon} + \gamma \cdot W_p(Q_Z^*,P_Z) \\
& \leq \sqrt[p]{\E_{X \sim P_X}\left[d(X,G(Q^*(X)))^p\right]} \\
&\qquad  + \gamma \cdot W_p(Q_Z^*,P_Z) + 2\gamma \varepsilon \\
& \leq \inf_{Q \text{ det.}} \bigg\{\sqrt[p]{\E_{X \sim P_X}\left[d(X,G(Q(X)))^p\right]}  \\
&\qquad  + \gamma \cdot W_p(Q_Z,P_Z) \bigg \} + \delta + 2\gamma\varepsilon \\
& \stackrel{\ref{mainth}}{\leq} W_p(P_X, P_G) + 3 \gamma \varepsilon,
\end{align*}
where in the last inequality we use additionally Formula \ref{rhs-corr}. As $\varepsilon$ is arbitrary we conclude that 
\begin{align*}
W_p(P_X,P_G) & =  \inf_{Q \in \mathcal{F}} \sqrt[p]{\E_{X \sim P_X}\E_{Z \sim Q(Z|X)}\left[d(X,G(Z))^p\right]} \\
& \qquad + \gamma \cdot W_p(Q_Z,P_Z)\, , 
\end{align*}
where $\mathcal{F}$ is any class of probabilistic encoders that at least contains a class of universal approximators.

Finally, $\mathcal{F}$ can be chosen as the set of all deterministic neural networks, indeed in \cite{hornik1991approximation} it is proven that they form a class of universal approximators for the L$_p$-norms $d(x,y)=\|x-y\|_p$ in Euclidean spaces.
\qed

\section{NOISE AS TARGETS (NAT)}\label{app:nat-compare} 

\cite{bojanowski2017unsupervised} introduce Noise As Targets (NAT), an algorithm for unsupervised representation learning. The method learns a neural network $f_{\theta}$ by embedding images into a uniform hypersphere. A sample $z$ is drawn from the sphere for each training image and fixed. The goal is to learn $\theta$ such that 1-to-1 matching between images and samples is improved: matching is coded with a permutation matrix $R$, and updated with the Hungarian algorithm. The objective is:
\begin{align}
\max_{\theta} \max_{R \in P_{M}}~\Tr(RZf_{\theta}(X)^\top),
\end{align}
where $\Tr(\cdot)$ is the trace operator, $Z$ and $X$ are respectively prior samples and images stacked in a matrix and $P_M \subset S_M$ is the set of $M$-dimensional permutations. NAT learns by alternating SGD and the Hungarian. One can interpret this problem as supervised learning, where the samples are targets (sampled only once) but their assignment is learned; notice that freely learnable $Z$ would make the problem ill-defined. The authors relate NAT to OT, a link that we make formal below.

We prove that the cost function of NAT is equivalent to ours when the encoder output is $L_2$ normalized, $c'$ is squared Euclidean and the Sinkhorn distance is considered with $\varepsilon=0$:
\begin{align}
& \argmax_{\theta} \max_{R \in P_{M}}~\Tr(RZf_{\theta}(X)^\top) \nonumber\\
& = \argmax_{\theta} \max_{R \in P_{M}}~\langle RZ, f_{\theta}(X) \rangle_F \nonumber\\
& = \argmin_{\theta} \min_{R \in P_{M}}~2 - 2 \langle RZ, f_{\theta}(X) \rangle_F \nonumber\\
& = \argmin_{\theta} \min_{R \in P_{M}}~ \| RZ\|_F^2 + \| f_{\theta}(X) \|_F^2\nonumber \\
& \qquad- 2 \langle RZ, f_{\theta}(X) \rangle_F \label{step1}\\
& = \argmin_{\theta} \min_{R \in P_{M}}~ \| RZ -  f_{\theta}(X) \|_F^2 \label{step2}\\
& = \argmin_{\theta} \min_{R \in P_{M}}~  \sum_{i,j} R_{i, j} \| z_i -  f_{\theta}(x_j) \|_2^2 \nonumber\\
& = \argmin_{\theta} \min_{R \in P_{M}}~  \langle R, C \rangle_F \nonumber \\
&  \subseteq \argmin_{\theta} \min_{R \in S_{M}}~  \tfrac{1}{M}\langle R, C \rangle_F - 0 \cdot H(R) \label{step3}\:\:.
\end{align}
Step \ref{step1} holds because both $R$ and $f_{\theta}(X)$ are row normalized. Step \ref{step2} exploits $R$ being a permutation matrix. The inclusion in Step \ref{step3} extends to degenerate solutions of the linear program that may not lie on vertices.
We have discussed several differences between our Sinkhorn encoder and NAT. There are other minor ones with \cite{bojanowski2017unsupervised}: ImageNet inputs are first converted to grey and passed through Sobel filters and the permutations are updated with the Hungarian only every 3 epochs. Preliminary experiments ruled out any clear gain of those choices in our setting.

%
%

\label{sec:visualizations}
\begin{figure*}[!b]
\vspace{-10pt}
\centering
\begin{subfigure}{.18\textwidth}
  \includegraphics[width=\linewidth]{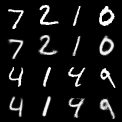}
\end{subfigure}
\begin{subfigure}{.3\textwidth}
  \includegraphics[width=\linewidth]{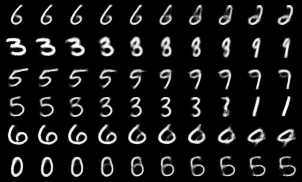}
\end{subfigure}
\begin{subfigure}{.3\textwidth}
  \includegraphics[width=\linewidth]{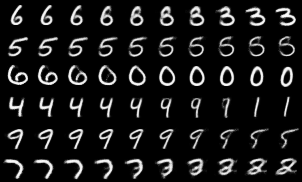}
\end{subfigure}
\begin{subfigure}{.18\textwidth}
  \includegraphics[width=\linewidth]{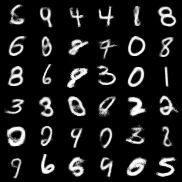}
\end{subfigure}
\begin{subfigure}{.18\textwidth}
  \includegraphics[width=\linewidth]{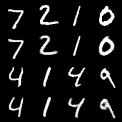}
\end{subfigure}
\begin{subfigure}{.3\textwidth}
  \includegraphics[width=\linewidth]{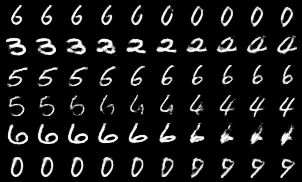}
\end{subfigure}
\begin{subfigure}{.3\textwidth}
  \includegraphics[width=\linewidth]{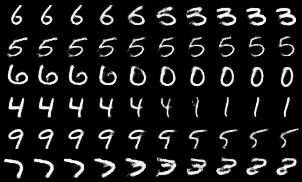}
\end{subfigure}
\begin{subfigure}{.18\textwidth}
  \includegraphics[width=\linewidth]{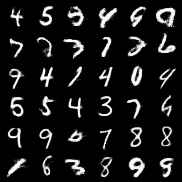}
\end{subfigure}
\caption{From left to right: MNIST reconstructions, extrapolations, interpolations, and samples. Models from Table \ref{table:vae}: VAE (top) and SAE (bottom).}
\label{fig:mnist_visualization}
\end{figure*}  

\begin{figure*}[!b]
\vspace{-10pt}
\centering
\begin{subfigure}{.49\textwidth}
  \includegraphics[width=\linewidth]{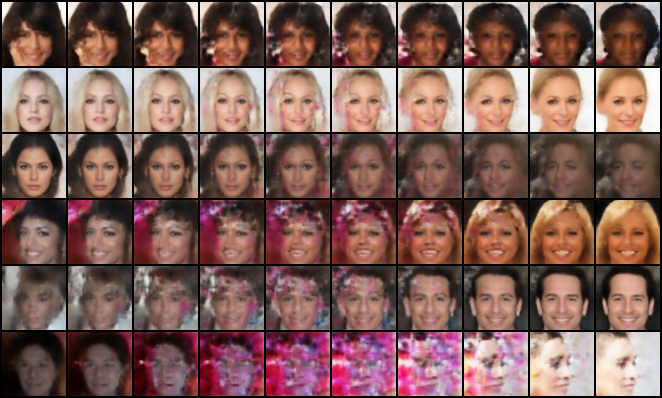}
\end{subfigure}
\begin{subfigure}{.49\textwidth}
  \includegraphics[width=\linewidth]{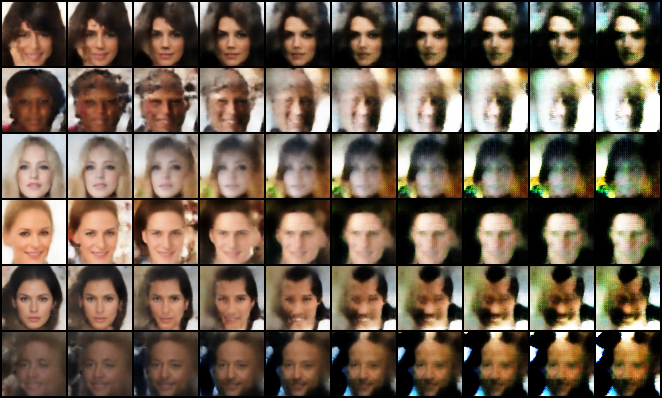}
\end{subfigure}
\begin{subfigure}{.25\textwidth}
  \includegraphics[width=\linewidth]{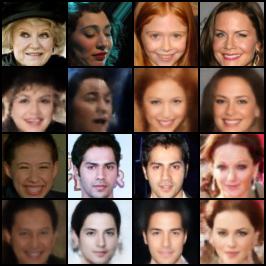}
 \end{subfigure}
  \begin{subfigure}{.25\textwidth}
  \includegraphics[width=\linewidth]{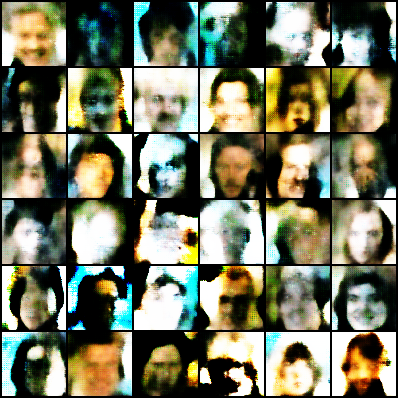}
\end{subfigure}
\caption{Celeba interpolations (top left), extrapolations (top right), reconstructions (bottom left) and samples (bottom right) for HAE with a Gaussian prior. The samples from the prior are clearly saturated, as can be expected from our hypothesis that the Hungarian algorithm for high-dimensional Gaussians causes the variance of the aggregated posterior to shrink.}
\label{fig:celeba_visualization}
\end{figure*}  

\begin{figure*}[h!]
\vspace{-10pt}
\centering
  \includegraphics[width=\linewidth]{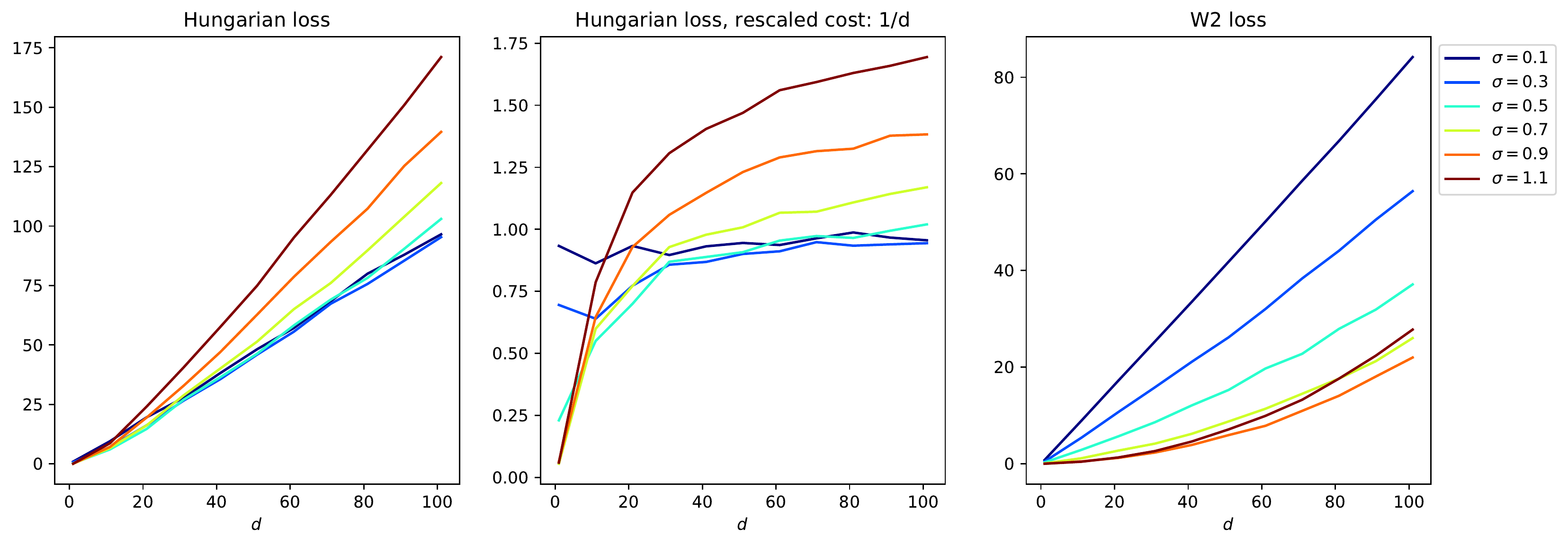}
\caption{Sample-matching Hungarian estimate of the optimal transport loss and the 2-Wasserstein loss based on sample estimates of the mean and covariances of two different sets of samples from two different Gaussians. Both sets of samples contain 128 samples, the first Gaussian has zero mean, and identity covariance, and the second Gaussian has zero mean and $\sigma$ times identity covariance. The losses are shown as a function of dimension $d$. }
\label{fig:high-d}
\end{figure*}  

\section{EXPERIMENTAL SETUP}\label{app:architectures}
We use similar architectures as in \citep{tolstikhin2018wasserstein}. For all methods a batchsize of 256 was used, except for HAE, which used a batch size of 128. HAE with batch size 256 simply becomes too slow. The learning rate fo all models was set to $0.001$, except for ($\beta$-)VAE, which used a learning rate of $0.0001$. FID scores for the CelebA dataset were computed based on the statistics following 
\url{https://github.com/bioinf-jku/TTUR} \citep{heusel2017gans} on the entire dataset. For MNIST we trained LeNet, and computed statistics of the first 55000 datapoints.

\section{VISUALIZATIONS OF AUTOENCODER RESULTS}
\label{app:highd}
In Figure \ref{fig:mnist_visualization} and \ref{fig:celeba_visualization} reconstructions and samples are shown for the MNIST and CelebA datasets.

\section{BEHAVIOUR OF SAMPLE-BASED OT LOSSES IN HIGH DIMENSIONS}
In Fig. \ref{fig:high-d} the sample-based estimated optimal transport cost according to the Hungarian algorithm and the 2-Wasserstein loss between two Gaussians with sample-based estimates for the mean and covariances are shown as a function of dimension. Two sets of samples are taken, one from a standard Gaussian (mean zero, identity covariance), and the second set is sampled from a zero-mean Gaussian with covariance $\sigma$ times the identity. As the dimension is increased, the Hungarian based OT estimates that samples from $\sigma <1$ match the samples from the standard normal distribution better than samples from $\sigma=1$. This can give rise to a reduced standard deviation of the encoder samples when combined with the Hungarian or Sinkhorn algorithm for matching in latent space to samples from a standard normal distribution. For the case of the 2-Wasserstein estimate with sample-based estimates for the mean and covariance this problem sets in only at much higher dimensions.

\end{document}